\documentclass{article}
\usepackage[preprint]{neurips_2023}

\usepackage{algorithm}
\usepackage{algpseudocode} 
\usepackage{microtype}
\usepackage{layouts}
\RequirePackage{times}
\RequirePackage{fancyhdr}
\RequirePackage{xcolor} 
\RequirePackage{algorithm}
\RequirePackage{natbib}
\RequirePackage{eso-pic} 
\RequirePackage{forloop}
\RequirePackage{url}
\usepackage[utf8]{inputenc} 
\usepackage[T1]{fontenc}    
\usepackage{url}            
\usepackage{booktabs}       
\usepackage{amsfonts}       
\usepackage{nicefrac}       
\usepackage{microtype}      
\usepackage{xcolor}         
\usepackage{multirow}
\usepackage{bm}
\usepackage{comment}
\usepackage{natbib}
\usepackage[tbtags]{amsmath}
\usepackage{amsthm}
\allowdisplaybreaks
\usepackage{amssymb,mathrsfs}
\usepackage{amsfonts}
\usepackage{upgreek}
\usepackage{xspace}
\usepackage{color}

\usepackage{bm}
\usepackage{wrapfig}
\usepackage{graphicx}
\usepackage{subcaption}
\usepackage{color}
\usepackage{stmaryrd}
\usepackage[inline]{enumitem}
\usepackage{url}

\usepackage{tikz}
\usetikzlibrary{calc}
\usepackage{bbm}
\usepackage{ifthen}
\usepackage{xargs}
\usepackage[textwidth=1.8cm]{todonotes}

\newcommand*\samethanks[1][\value{footnote}]{\footnotemark[#1]}

\usepackage{titlesec}

\usepackage{hyperref}
\hypersetup{
  colorlinks,
  linkcolor={red!50!black},
  citecolor={blue!50!black},
  urlcolor={blue!80!black}
}

\makeatletter

\definecolor{oxprimary}{HTML}{002147}
\definecolor{oxsecondry}{HTML}{a79d96}
\definecolor{oxtertiary}{HTML}{f3f1ee}
\definecolor{oxlightprimary}{HTML}{122f53}
\definecolor{oxverylightblue}{HTML}{f0f5f8}

\definecolor{oxblack}{HTML}{000000}
\definecolor{oxveryoffblack}{HTML}{333333}
\definecolor{oxmidgrey}{HTML}{7a736e}
\definecolor{oxdarkgrey}{HTML}{a6a6a6}
\definecolor{oxlightgrey}{HTML}{e0ded9}
\definecolor{oxvlightgrey}{HTML}{f9f8f5}
\definecolor{oxwhite}{HTML}{ffffff}

\definecolor{tabblue}{HTML}{1f77b4}
\definecolor{taborange}{HTML}{ff7f0e}
\definecolor{tabgreen}{HTML}{2ca02c}
\definecolor{tabred}{HTML}{d62728}
\definecolor{tabpurple}{HTML}{9467bd}
\definecolor{tabbrown}{HTML}{8c564b}
\definecolor{tabpink}{HTML}{e377c2}

\definecolor{brightgrey}{HTML}{f7f7f7}

\usepackage[most]{tcolorbox}
\tcbuselibrary{theorems}
\tcbuselibrary{breakable}
\usepackage[capitalise]{cleveref}
\crefname{theorem}{Theorem}{Theorems}
\Crefname{theorem}{Theorem}{Theorems}
\crefname{proposition}{Proposition}{Propositions}
\Crefname{proposition}{Proposition}{Propositions}
\crefname{lemma}{Lemma}{Lemmas}
\Crefname{lemma}{Lemma}{Lemmas}
\crefname{corollary}{Corollary}{Corollaries}
\Crefname{corollary}{Corollary}{Corollaries}
\crefname{definition}{Definition}{Definitions}
\Crefname{definition}{Definition}{Definitions}
\crefname{assumption}{Assumption}{Assumptions}
\Crefname{assumption}{Assumption}{Assumptions}
\crefname{remark}{Remark}{Remarks}
\Crefname{remark}{Remark}{Remarks}

\newenvironment{talign*}
 {\csname align*\endcsname}
 {\endalign}

\def\x{\mathbf{x}}
\def\bfx{\mathbf{x}}
\def\rset{\mathbb{R}}
\def\E{\mathbb{E}}
\def\rmd{\mathrm{d}}

\def\bfY{\mathbf{Y}}

\def\bfB{\mathbf{B}}
\def\bfZ{\mathbf{Z}}
\def\calF{\mathcal{F}}
\def\calD{\mathcal{D}}
\newcommand{\coint}[1]{\left[#1\right)}
\newcommand{\ball}[2]{\operatorname{B}(#1,#2)}
\def\bfa{\mathbf{a}}
\def\mcx{\mathcal{X}}

\newcommand{\PE}{\mathbb{E}}
\newcommand{\expe}[1]{\mathbb{E}[#1]}
\def\Id{\operatorname{Id}}
\newcommand{\ccint}[1]{\left[#1\right]}
\def\nset{\mathbb{N}}
\newcommand{\ensembleLigne}[2]{\{#1\,:\eqsp #2\}}
\newcommand{\Pens}{\mathscr{P}}
\def\rmS{\mathrm{S}}
\def\calX{\mathcal{X}}
\newcommand{\expeLigne}[1]{\PE [ #1 ]}
\def\rmc{\mathrm{C}}
\newcommand{\rme}{\mathrm{e}}
\newcommand{\R}{\mathbb R}

\def\eqsp{}
\newcommandx{\normLigne}[2][1=]{\ifthenelse{\equal{#1}{}}{\Vert #2 \Vert}{\Vert #2\Vert^{#1}}}
\def\rmK{\mathrm{K}}
\def\rmT{\mathrm{T}}
\newcommandx{\KL}[2]{\operatorname{KL}( #1 | #2 )}

\newcommand{\argmin}{\operatornamewithlimits{argmin}}
\def\rmL{\mathrm{L}}
\def\vareps{\varepsilon}

\usepackage[toc,page,header]{appendix}
\usepackage{natbib} 
    
\usepackage{aliascnt}

\theoremstyle{plain}
\newtheorem{theorem}{Theorem}[section]

\newaliascnt{proposition}{theorem}
\newtheorem{proposition}[proposition]{Proposition}
\aliascntresetthe{proposition}

\newaliascnt{lemma}{theorem}

\aliascntresetthe{lemma}

\newaliascnt{corollary}{theorem}

\aliascntresetthe{corollary}

\theoremstyle{definition}
\newaliascnt{definition}{theorem}

\aliascntresetthe{definition}

\newaliascnt{assumption}{theorem}

\aliascntresetthe{assumption}

\theoremstyle{remark}
\newaliascnt{remark}{theorem}
\newtheorem{remark}[remark]{Remark}
\aliascntresetthe{remark}

\usepackage{layouts}
\newcommand{\appendixhead}{
  \centerline{\textbf{\LARGE Supplementary to: }\vspace{0.15in}}
  \centerline{\textbf{\LARGE Spectral Diffusion Processes}\vspace{0.25in}}
}
\usepackage[tableposition=top]{caption}

\titlespacing*{\paragraph}{0pt}{0pt}{0.5em}
\titlespacing*{\section}{0pt}{*0.8}{*0.6}
\titlespacing*{\subsection}{0pt}{*0.7}{*0.5}

\title{Spectral Diffusion Processes}

\author{\textbf{Angus Phillips}\thanks{Correspondence to: \texttt{angus.phillips@stats.ox.ac.uk}.}\;\;\thanks{Dept. of Statistics, University of Oxford, Oxford, UK.}\ ,\; \textbf{Thomas Seror}\thanks{Dept. of Computer Science ENS, CNRS, PSL University Paris, France.}\ ,\; \textbf{Michael Hutchinson}\samethanks[2]\ ,\; \textbf{Valentin De Bortoli}\samethanks[3]\ , \\ \textbf{Arnaud Doucet}\samethanks[2]\ \textbf{,}\;  \textbf{\'Emile Mathieu}\thanks{Dept. of Engineering, University of Cambridge, Cambridge, UK.} }

\begin{document} 
\maketitle
\setcounter{footnote}{0}

\begin{abstract}
  \looseness=-1 Diffusion models have proven to be a flexible and effective framework for modelling probability distributions on finite-dimensional spaces. However, many physical modelling problems such as time series are naturally described over function spaces. In this work we apply diffusion models to such stochastic processes. To do so we consider a spectral representation of the data, obtained using a kernel, thereby dissociating the stochastic part of the processes from their space-time structure. As a result, the stochasticity of the processes is entirely encoded in the spectral coefficients, which we truncate and model using standard finite-dimensional diffusion models. By truncating the representation in the spectral domain we ensure our resulting model defines valid stochastic processes, thereby naturally satisfying consistency and exchangeability criteria. Projecting our spectral diffusion models back to the original input space, we show that for any given marginals our approach corresponds to a diffusion model with \emph{correlated} noise, with explicit covariance matrix given by the kernel. We demonstrate our method's effectiveness for modelling various multimodal datasets as well as conditional sampling by amortising our models with respect to a context set.
\end{abstract}

\section{Introduction}
\label{sec:introduction}

Probabilistic modelling over functional spaces is of high importance in the physical sciences where objects of interest are often \emph{fields}, mapping (space-time) inputs to tensors such as scalars or vectors. Examples range from weather forecasting~\citep{ravuri2021Skilful} and prediction of electrical potentials~\citep{yang2020physics}, to epidemiology \citep{wu2022Multifidelity}. The desire to capture uncertainty motivates a \emph{probabilistic} treatment of these phenomena, therefore requiring the development of flexible parametric probabilistic models on functional spaces. Of particular importance is the modelling of \emph{conditional} distributions given a set of observations, such as the probability distribution over paths of a hurricane given its recent trajectory \citep{giffard-roisin2021Deep}.

However, two challenges arise when constructing parametric probabilistic models over functional spaces. The first is a problem of validity: finite representations do not necessarily imply \emph{valid} distributions over stochastic processes---they must satisfy the consistency and exchangeability requirements of the Kolmogorov extension theorem \cite{charalambos2013infinite} in order to do so. Of the two, consistency is harder to achieve in practice, yet without it the model's predictive distribution at an arbitrary location depends on the resolution of the discretisation. The second is a problem of flexibility: parametric distributions over function spaces must have the flexibility to capture complex non-Gaussian, multimodal behaviour. Amongst the existing suite of approaches, these requirements are typically at odds with each other. In particular, a well-established framework for modelling stochastic processes is Gaussian process (GP) regression~\citep{rasmussen2003gaussian}, which allows for closed form posterior prediction on functional spaces. GPs place Gaussian distributions on the finite dimensional marginals of the process, with the covariance structure specified by an arbitrary kernel. While their Gaussian construction means that GPs satisfy consistency and exchangeability, that same construction restricts the flexibility of the resulting model. Furthermore, specification of complex kernel functions is challenging in practice and their computational cost of $\mathcal{O}(N^3)$ in the number of data points $N$ is often prohibitive. Neural processes (NPs)~\citep{garnelo2018neural,garnelo2018Conditional} aim to improve the flexibility of GPs by meta-learning the finite dimensional marginals directly from data using deep neural networks. This family of models has been successfully applied on a wide variety of probabilistic modelling tasks, such as temperature and precipitation predictions \citep{vaughan2022Convolutional,foong2020MetaLearning}, epidemiology modelling \citep{wu2022Multifidelity} and photometric time series modelling \citep{gordon2019Convolutional}. However, in exchange for increased flexibility NPs only satisfy \emph{conditional} consistency, and therefore they do not define a valid distribution over unconditional stochastic processes.

\begin{figure}[t]
    \centering
    \includegraphics[width=.82\textwidth]{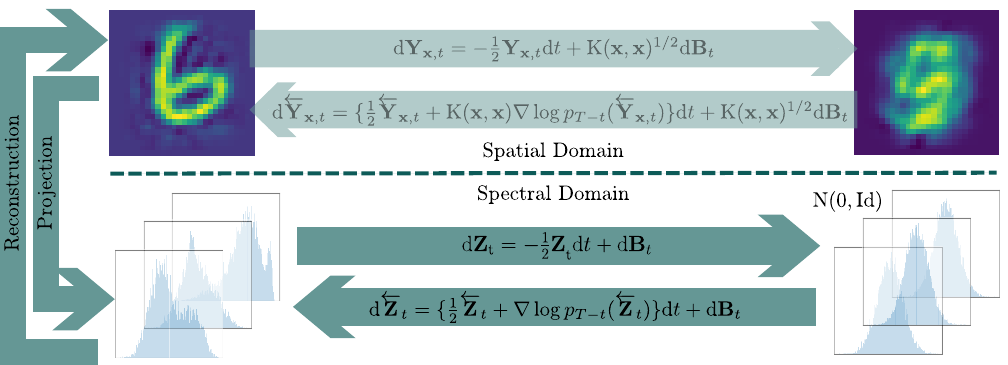}
    \caption{Illustration of our methodology. Stochastic processes are parameterised via a functional basis with stochastic coefficients. A diffusion model is learnt on spectral coefficients $\bfZ$, with the processes $(\bfZ_t)_{t\geq0}$ and $(\overleftarrow{\bfZ})_{t\geq0}$ in spectral space inducing the spatial processes $(\bfY_t)_{t\geq0}$ and $(\protect\overleftarrow{\bfY}_t)_{t\geq0}$ via the reconstruction with the functional basis. In the spatial domain $p_t$ is the density of $\bfY_{x,t}$ and in the spectral domain $p_t$ is the density of $\bfZ_t$.}
    \label{fig:killer_plot}
\end{figure}

Denoising diffusion models \citep{song2019generative,song2020score,ho2020denoising,nichol2021beatgans} have recently been introduced as a new powerful paradigm for generative modelling.  They work as follows: noise is progressively added to data following a Stochastic Differential Equation (SDE)---the forward \emph{noising} process---until it is approximately Gaussian. The generative model is given by an approximation of the associated \emph{time-reversed process} called the \emph{denoising} process. Recent works have extended denoising diffusion models to model distributions over functions \citep{dutordoir2022neural,kerrigan2022Diffusion,pidstrigach2023infinite,lim2023score,bond2023infty,hagemann2023multilevel,franzese2023continuous}. While these distributions are extremely flexible, infinite-dimensional diffusion approaches are discretised in practice and the resulting finite-dimensional models do not satisfy consistency. 

To address these shortcomings, we present a flexible probabilistic model which is \emph{by construction} a valid distribution over stochastic processes. We do so by decomposing stochastic processes onto an eigenbasis associated with a kernel $\rmK$, thereby decoupling the \emph{stochastic} part of the process from its \emph{space-time} part. This gives a \emph{countable} representation of an \emph{uncountable} problem (on a compact space), which we in turn restrict to a finite representation by a principled truncation. We finally learn a flexible probability distribution over the finite stochastic coefficients with a denoising diffusion model. Our key insight is that, by choosing to truncate the representation in the eigenbasis rather than the input space, we obtain a finite parameterisation which naturally defines a distribution over stochastic processes, from which consistency and exchangeability are trivial consequences. Projecting our diffusion model on spectral coefficients back into the original space we show that, for any given spatial input, our process corresponds to a diffusion model with \emph{correlated} noise, thereby drawing a bridge with existing approaches such as \citet{kerrigan2022Diffusion,lim2023score}. We further tackle conditional modelling tasks by taking inspiration from neural processes and amortising the score network over conditioning observations with a permutation invariant encoder. Our methodology is illustrated in \Cref{fig:killer_plot} and presented in more detail in \cref{sec:method}. We empirically demonstrate the flexibility of our approach on several unconditional and predictive modelling tasks in \cref{sec:experiments}.

\section{Background}
\label{sec:background}

\paragraph{Stochastic processes.}
In this work we are interested in modelling distributions over functions, which we formalise through the framework of stochastic processes.  Given a compact input space $\calX$, an $\rset^d$-valued stochastic process $(\bfY_x)_{x \in \calX}$ is a collection of random variables $\bfY_x$ taking values in \(\rset^d\).\footnote{In this work, we focus on \emph{continuous} stochastic processes, meaning any trajectory of $(\bfY_x)_{x \in \calX}$ belongs to $\rmc(\calX, \rset^d)$ so that $\omega \mapsto (\bfY_x(\omega))_{x \in \calX}$ is measurable and we denote $\pi \in \mathcal{P}(\rmc(\calX, \rset^d))$ its distribution.} Typically, existing models of stochastic processes parameterise \emph{finite dimensional marginals} $\ensembleLigne{\bfY_{x_i}}{i \in \{1, \dots, n\}, \ x_i \in \calX}$ for every $n \in \nset$, see~\citet{garnelo2018neural} for example. However, such a collection of probability distributions $\rmS = \ensembleLigne{\pi_{x_1, \dots, x_n}}{(x_1, \dots, x_n) \in \calX^n, \ n\in \nset}$ is not guaranteed to define a distribution $\pi \in \mathcal{P}((\rset^d)^\calX)$\footnote{$\mathcal{P}((\rset^d)^\calX)$ is the space of distributions over the space $(\rset^d)^\calX$.} over stochastic processes. The conditions under which this is the case are given by the \emph{Kolmogorov extension theorem}.

\begin{theorem}[Kolmogorov Extension Theorem \citep{charalambos2013infinite}]
\label{thm:ket}
  Given a collection of finite dimensional marginals $\rmS = \ensembleLigne{\pi_{x_1, \dots, x_n}}{(x_1, \dots, x_n) \in \calX^n, n \in \nset}$ over the index set \(\calX\), there exists a unique measure $\pi$ over $(\rset^d)^\calX$ whose marginals match $\rmS$ if and only if the collection of marginals satisfies the following conditions:
    \begin{enumerate}[label=\alph*),wide]
    \item \emph{exchangeability:} for any $n \in \nset$, $n$-permutation $\sigma \in \mathfrak{S}_n$, $(x_1, \dots, x_n) \in \calX^n$ and continuous and bounded function $f \in \rmc_b((\rset^d)^n)$, we have
    \begin{align*}
        \textstyle \int f(y_{\sigma^{-1}(1)},
        \dots, y_{\sigma^{-1}(n)}) \rmd \pi_{x_{\sigma(1)}, \dots, x_{\sigma(n)}}(y_1, \dots,
        y_n)
         = \int f(y_{1}, \dots, y_{n}) \rmd \pi_{x_1,
          \dots, x_n}(y_1, \dots, y_n) .
    \end{align*}
        \item \emph{consistency:} for any $n_1 \leq n_2$, $(x_1, \dots, x_{n_2}) \in \calX^{n_2}$ and $f \in \rmc_b((\rset^d)^{n_1})$ we have
    \begin{align*}
    \textstyle \int f(y_1, \dots, y_{n_1}) \rmd \pi_{x_1, \dots, x_{n_1}}(y_1, \dots, y_{n_1})
     = \int f(y_1, \dots, y_{n_1}) \rmd \pi_{x_1, \dots, x_{n_2}}(y_1, \dots, y_{n_2}).
    \end{align*}
  \end{enumerate}
\end{theorem}
Existing functional generative modelling approaches \citep[e.g.][]{dutordoir2022neural,kerrigan2022Diffusion} often satisfy the exchangeability condition but typically fall short on consistency, while neural processes~\citep{garnelo2018neural} only satisfy consistency in specific settings. In this work we bridge this gap by developing a \emph{spectral} method which yields a valid generative model of a stochastic process \emph{by definition}. 

\paragraph{Denoising diffusion models.} \label{sec:SGMs}
We briefly recall the concepts behind denoising diffusion models on the Euclidean
space $\rset^d$. Let $(\bfZ_t)_{t \geq 0}$ be a \emph{noising} process defined by the
following SDE:
\begin{equation}\label{eq:forward_SDE}
  \rmd \bfZ_t = -\tfrac{1}{2}\beta_t \bfZ_t \,\rmd t + \sqrt{\beta_t}\,\rmd \bfB_t,
  \quad \bfZ_0 \sim p_0,
\end{equation}
with $(\bfB_t)_{t \geq 0}$ a $d$-dimensional Brownian motion, $p_0$ the data
distribution on $\rset^d$, and $\beta_t > 0$ a noise schedule controlling the rate at
which noise is injected. In the infinite time limit this process converges towards
the unit Gaussian distribution $\mathrm{N}(0, \Id)$. Note that here we describe the
VP-SDE framework~\citep{song2020score} (the continuous time limit of
DDPM~\citep{ho2020denoising}), but changing the forward process yields alternative
frameworks~\citep{song2020score}. Under conditions on $p_0$, the time-reversal
$(\overleftarrow{\bfZ}_t)_{t \in \ccint{0,T}} = (\bfZ_{T-t})_{t \in \ccint{0,T}}$ is
given by \citet{cattiaux2021time,haussmann1986time}
\begin{equation}
\label{eq:backward-SDE}
  \rmd \overleftarrow{\bfZ}_t = \beta_{T-t}\{ \tfrac{1}{2}\overleftarrow{\bfZ}_t
    + \nabla \log p_{T-t}(\overleftarrow{\bfZ}_t)\} \,\rmd t
    + \sqrt{\beta_{T-t}}\,\rmd \bfB_t, \quad \overleftarrow{\bfZ}_0 \sim p_T,
\end{equation}
with $p_t$ the density of $\bfZ_t$. Denoising diffusion models are defined to sample
approximately from \eqref{eq:backward-SDE}. To do so, we replace the finite time
distribution $p_T$ by $\mathrm{N}(0,\Id)$ and approximate the intractable Stein score
$\nabla \log p_t$ by a parametric function $s_\theta(t, \cdot)$ which is trained using
score matching techniques~\citep{hyvarinen2005estimation,vincent2011connection,song2019Sliced}.
With these approximations in place, an Euler--Maruyama discretisation of
\eqref{eq:backward-SDE} is applied using a step-size $\gamma > 0$ and $N=T/\gamma$:
\begin{equation*}
\label{eq:backward_discrete_final}
 \tilde{\bfZ}_{n+1} = \tilde{\bfZ}_n + \gamma\, \beta_{T-n\gamma} \{\tfrac{1}{2}\tilde{\bfZ}_n
   + s_\theta(T -n \gamma, \tilde{\bfZ}_n) \} + \sqrt{\gamma\,\beta_{T-n\gamma}}\, \mathbf{G}_{n+1},
   \ \tilde{\bfZ}_0\sim \mathrm{N}(0,\Id),\ \mathbf{G}_{n+1} \overset{\textup{i.i.d.}}{\sim} \mathrm{N}(0,\Id).
\end{equation*}
Simulating this recursion yields approximate samples $\tilde{\bfZ}_N$ from the data
distribution $p_0$.

\section{Spectral Diffusion Processes}
\label{sec:method}

Our aim is to define a generative model for stochastic processes that is both \emph{valid} and fully flexible---capable, for instance, of capturing non-Gaussian, multimodal behaviour in function space. Achieving both these goals has thus far been elusive due to the inherent difficulty that a stochastic process on a compact input space $\mathcal{X}$ is an uncountable object that cannot be represented or learned directly; any practical model must resort to a finite representation of it. Typically this is achieved by discretising the process and modelling its finite-dimensional marginals. Such a choice is expressive, but the resulting marginals are typically not coherent, thus violating the \emph{consistency} condition of the Kolmogorov extension theorem (KET) and failing to define a valid process.

Our key insight is that the limitations of existing approaches lie not in using a finite representation, but in \emph{which} representation is chosen. Rather than discretising in the input space, we propose to represent the process through its coefficients in a countable functional basis and we perform truncation in this countable basis. Importantly, the truncated object remains a valid stochastic process by definition, while flexibility is inherited from the generative model on stochastic coefficients. This is made precise in the following section: we obtain the basis via spectral theory (\Cref{sec:mercer}) and we model the truncated coefficients with a diffusion model (\Cref{sec:model}), ultimately resulting in a valid and principled generative model for stochastic processes. 

\subsection{Spectral Representation of Stochastic Processes}
\label{sec:mercer}

\paragraph{Mercer's theorem.} Our approach is based on constructing a spectral representation of a stochastic process. We use Mercer's theorem (\cref{thm:mercer}), which enables a \emph{countable} representation of an \emph{uncountable} process. We start by recalling this central result.\footnote{We present our methodology for processes taking values in $\mathbb{R}$ for simplicity, but note that our derivations extend to the $\mathbb{R}^d$ case, see Appendix~\ref{app:rd_processes}.}

\begin{theorem}[Mercer's Theorem \citep{ferreira2009eigenvalues}]
\label{thm:mercer}
  Let $\mathcal{X}$ be a compact metric space and let
  $\rmK: \ \mcx \times \mcx \to \rset$ be a continuous, symmetric, positive definite kernel, i.e. for any $n \in \nset$ and $\{x_i\}_{i=1}^n \in \calX^n$ we have $c^\top \rmK(\{x_i\}_{i=1}^n,\{x_i\}_{i=1}^n)c > 0$ for all non-zero $c \in \mathbb{R}^n$. 
  Define a self-adjoint compact operator $\rmT_{\rmK}$ on $\mathrm{L}^2(\calX)$ by
  \begin{equation} \label{eq:operator_kernel_definition} \textstyle
    \rmT_{\rmK} [f](\cdot) = \int_\calX \rmK(x,\cdot) f(x) \rmd
    x\end{equation} for any $f \in \mathrm{L}^2(\calX)$. Then, there exist
  $(e_m)_{m \in \nset} \in (\mathrm{L}^2(\calX))^\nset$ and
  $(\lambda_m)_{m \in \nset} \in (0, \infty)^\nset$ such that $(e_m)_{m \in \nset}$ is
  an orthonormal basis of $\mathrm{L}^2(\calX)$ and
  $(e_m, \lambda_m)_{m \in \nset}$ is an eigensystem for $\rmT_\rmK$. Furthermore, $(e_m)_{m \in \mathbb{N}}$ consists of continuous functions only and the expansion
  \begin{equation}
      \textstyle \rmK(x, x') = \sum_{m=0}^\infty \lambda_m e_m(x) e_m(x')
  \end{equation}
  converges absolutely and uniformly on $\mathcal{X}\times \calX$. 
\end{theorem}

We now leverage the basis provided by Mercer's theorem to represent stochastic processes in spectral space. We will henceforth assume the stochastic process $(\bfY_x)_{x \in \mathcal{X}}$ has finite second moment, $\expeLigne{\normLigne{\bfY_x}^2} < +\infty$, and without loss of generality we also assume $\expeLigne{\bfY_x} = 0$ (the mean may be subtracted otherwise). Finiteness of the second moment ensures each realisation of the process lies in $L^2(\calX)$, therefore we can exactly expand $\bfY_x$ using the orthonormal basis $(e_m)_{m \in \nset}$, 
\begin{equation}
  \textstyle
  \bfY_x = \sum_{m=0}^{\infty} \lambda_m^{1/2} Z_m\, e_m(x), \qquad
  Z_m = \lambda_m^{-1/2} \int_\calX \bfY_x e_m(x)\, \rmd x ,
\end{equation}
with the series converging in $L^2(\calX)$. The scaling $\lambda_m^{1/2}$ is a simple normalisation of the coefficients which we justify shortly. We remark that this representation is exact but infinite-dimensional. To now obtain a \emph{finite} representation we truncate the expansion at order $M$, defining
\begin{equation} \label{eq:parameterization}
  \textstyle \bfY^M_x \triangleq \sum_{m=0}^M \lambda_m^{1/2} Z_m\, e_m(x) .
\end{equation}
This truncation is principled because $\rmT_\rmK$ is trace-class---$\sum_m \lambda_m = \int_\mathcal{X}\rmK(x, x) \, \rmd x < + \infty$---and $\lambda_m \to 0$~\citep{ferreira2009eigenvalues}, meaning the discarded tail contains an arbitrarily small fraction of the spectral energy.
  
Using the spectral representation \eqref{eq:parameterization}, we dissociate the \emph{stochastic} part of the process---embedded in the stochastic coefficients $\{Z_m\}_{m=0}^{M} \in \rset^{M+1}$---from its \emph{space-time} part---embedded in the deterministic basis $\{e_m\}_{m=0}^M \in (\mathrm{L}^2(\calX))^{M+1}$. We believe this parameterisation is particularly appealing since (a) Mercer's kernels define Reproducing Kernel Hilbert Spaces (RKHS) which can be rich functional spaces (e.g.\ universal RKHS are dense in the space of bounded continuous functions \citep[Section 4.6]{steinwart2008support}), (b) prior knowledge of the data process can be incorporated in the choice of kernel (e.g.\ a kernel's bandwidth regulates the smoothness of functions). Our approach is therefore expressive and encompasses a variety of popular functional bases, such as the Fourier basis or wavelet bases.

\paragraph{Karhunen-Lo\`{e}ve extension.} While in practice we implement our method using general continuous kernels via Mercer's theorem, choosing the \emph{covariance kernel} $\rmK_\bfY(x, x') = \mathbb{E}\left[\bfY_x \bfY_{x'}\right]$ provides additional guarantees via the Karhunen-Lo\`{e}ve theorem, which we state fully in Appendix~\ref{app:KL_theorem}. In particular, under the covariance kernel the spectral coefficients satisfy $\operatorname{Cov}(Z_m, Z_{m'}) = \delta_{m, m'}$ and
\begin{equation}
\label{eqn:kl_decomp_y}
\textstyle
{\lim\limits_{M \to +\infty} \mathbb{E}[\sup_{x \in \calX} \|\bfY_x - \sum_{m=0}^M {\lambda_m}^{1/2} Z_m e_m(x)\|^2] = 0,
}
\end{equation}
giving stronger uniform-in-$x$ mean-squared convergence of our representation. 

\paragraph{Validity of the representation.} We emphasise that our finite representation in \Cref{eq:parameterization} is \emph{by definition} a stochastic process. In particular, for a fixed $x \in \calX$ it is a finite linear combination of the measurable $Z_m$ with \emph{deterministic} coefficients ${\lambda_m}^{1/2} e_m(x)$ and thus is itself a measurable function $\omega \mapsto Y_x(\omega)$. We therefore satisfy consistency and exchangeability \emph{for free}, by the trivial direction of the KET. For illustrative purposes, we also provide a proof of consistency and exchangeability in our setting in Appendix~\ref{app:consistency_exchangability_proof}. Our method therefore defines stochastic processes directly, whereas non-consistent approaches must learn consistency from the data~\citep{dutordoir2022neural}.

\subsection{Spectral Diffusion Model}
\label{sec:model}

The spectral representation $\bfY_x^M$ derived in \Cref{sec:mercer} reduces the problem of modelling a distribution over functions to that of modelling the distribution of the finite coefficient vector $\bfZ = (Z_0, \dots, Z_M)$, which we choose to do using a finite-dimensional diffusion model. When mapped back to function space through the basis, this induces a flexible model over stochastic processes. We refer to \Cref{fig:killer_plot} for an illustration of our methodology. Training the diffusion model requires a pre-processing step which converts a stochastic process dataset into spectral coefficients $\{\bfZ^s\}_{s=1}^S$ for the $S$ processes in the dataset---we detail the required computations in \Cref{sec:compute_bases}.

\paragraph{Spectral diffusion model.}
We model the distribution of the coefficient vector $\bfZ \in \rset^{M+1}$ with a standard finite-dimensional diffusion model. We use the VP-SDE formulation introduced in \Cref{sec:SGMs} to define the forward process $(\bfZ_t)_{t \in [0, T]}$ and reverse process $(\overleftarrow{\bfZ}_t)_{t \in [0, T]}$, where $p_0$ is the distribution of the spectral coefficients that corresponds to the distribution $\pi \in \mathcal{P}(\mathrm{C}(\calX, \mathbb{R}))$ of the stochastic process. Samples from the spectral dataset are distributed $\bfZ^s \sim p_0$ and are used to train the score network $\bm{s}_\theta: \rset^{M+1} \times [0,T] \rightarrow \R^{M+1}$ by minimising the denoising score matching (DSM) loss \citep{vincent2011connection}:
\begin{equation} \label{eq:dsm_standard}
\textstyle{\mathbb{E}_t \{ \mathbb{E}_{\bfZ_0, \bfZ_t} [ \|\bm{s}_\theta(\bfZ_t,t) - \nabla_{\bfZ_t} \log p(\bfZ_t|\bfZ_0) \|^2_2  ] \} \eqsp ,}
\end{equation}
with $t \sim U([0, T])$. Sampling the time-reversal initialised from $\overleftarrow{\bfZ}_0 \sim \mathrm{N}(0, \Id)$ yields coefficients $\overleftarrow{\bfZ}_T = (\overleftarrow{Z}_{0,T}, \dots, \overleftarrow{Z}_{M,T})$ approximately distributed according to $p_0$. When recombined with the basis, this defines our generative model in function space,
\begin{equation} \label{eq:reconstruction} \textstyle
  \overleftarrow{\bfY}^M_T = \sum_{m=0}^M \lambda_m^{1/2}\, \overleftarrow{Z}_{m,T}\, e_m ,
\end{equation}
which can be evaluated at any input $x \in \calX$. Unlike the Gaussian likelihoods of Gaussian and neural processes, the diffusion model imposes no parametric form on the coefficients, so it can capture arbitrary multimodal, non-Gaussian distributions. This is precisely the object we set out to construct: a flexible generative model over functions---inheriting the expressivity
of the diffusion---that is a valid stochastic process by construction (\Cref{sec:mercer}).

\paragraph{Conditional sampling.}
We additionally extend our approach to enable predictive modelling of the posterior $p(y^* \mid x^*, \{x^i, y^i\}_{i \in C})$ given a context set $\{x^i, y^i\}_{i \in C}$. Following the neural-process paradigm \citep{garnelo2018neural}, we amortise the score network over context sets, $\bm{s}_\theta(\cdot, t, c) \approx \nabla \log p_t(\cdot \mid \{x^i, y^i\}_{i \in C})$, where $c = \mathrm{enc}_\theta(\{x^i, y^i\}_{i \in C})$ is a permutation-invariant embedding of the context set.  In practice we use the self-attention architecture of \citet{dutordoir2022neural} as our encoder. Augmenting the score network with the embedding $c$ allows the coefficient distribution, and the induced distribution over processes, to adapt to the conditioning observations. However, similarly to neural processes, this approach does not explicitly enforce the predictive distribution to interpolate the context, that is $p(y^c \mid x^c, y^c)$ is not necessarily $\delta_{y^c}$. In practice the model is found to accurately learn this behaviour from data, giving $\E[\bfY_{x^c} \mid x^c, y^c] \approx y^c$ (\Cref{fig:conditional_gp_dataset}).

\subsection{Properties of Spectral Diffusion Processes}
\label{sec:method_properties}

Our generative model over stochastic processes is naturally defined in the \emph{spectral} space by modelling the coefficients of the stochastic process in the basis $(e_m)_{m \in \mathbb{N}}$. In particular, we use a diffusion model on coefficients $\bfZ$ with forward process $(\bfZ_t)_{t \geq 0} = \{(Z_{m,t})_{t \geq 0} \}_{m=0}^M$. The diffusion process in spectral space induces an associated \emph{spatial} counterpart by pushing the process through the basis decomposition. We define this induced spatial process $\bfY^M_{\bfx,t} \triangleq \sum_{m=0}^M \lambda_m^{1/2} Z_{m,t} e_m(\bfx)$ where $\bfx \in \calX^n$ for some $n \in \nset$. We now present two results which help interpret and justify our proposed methodology. These results are visually illustrated in \Cref{fig:mnist_trajectories}.

Firstly, we note that the base distribution $\overleftarrow{\bfY}_0^M = \sum_{m=0}^M \lambda_m^{1/2} \overleftarrow{Z}_{m,0} e_m, \; \overleftarrow{Z}_{m,0} \sim \mathrm{N}(0, \Id)$, is a Gaussian process (GP) and ---when using the covariance kernel $\rmK=\rmK_\bfY$--- we can show that $\overleftarrow{\bfY}_0 = \sum_{m=0}^{+\infty} \lambda_m^{1/2} \overleftarrow{Z}_{m,0} e_m$ is the closest GP to the target distribution $\pi$ in the following sense (see \cref{sec:proof-prop-limiting} for a proof of~\cref{prop:proj_gaussian_process}.)

\begin{proposition}[Limiting process]
  \label{prop:proj_gaussian_process}
  Let $\bfY$ be a stochastic process with law $\pi$. Let $(\lambda_m, e_m)_{m\in \nset}$ be the eigenbasis of the Karhunen-Lo\`{e}ve decomposition of $\bfY$ (Appendix~\ref{app:KL_theorem}). Let $\pi^0$ be the distribution of $\overleftarrow{\bfY}_0=\sum_{m=0}^{+\infty} \lambda_m^{1/2} \overleftarrow{Z}_{m, 0} e_m$, $\overleftarrow{Z}_{m,0}\sim \mathrm{N}(0, \Id)$. Denote $\mathrm{GP}(\calX)$ the space of Gaussian processes on $\calX$. Then $\pi^0 = \argmin_{\pi_{\mathrm{GP}} \in \mathrm{GP}(\calX)} \KL{\pi}{\pi_{\mathrm{GP}}}$. In addition,  $\overleftarrow{\bfY}_0^M$ is the projection of 
  $\overleftarrow{\bfY}_0$ on the subspace of $\rmL^2(\calX)$ spanned by $\{e_m\}_{m=0}^M$.
\end{proposition}

Secondly, we prove that the induced forward spatial process corresponds to an Ornstein-Uhlenbeck process with \emph{correlated} noise, i.e. the Brownian motion $\bfB_t$ is replaced by $\rmK(\bfx, \bfx)^{1/2} \bfB_t$ where $\x\in \calX^n$ and $n \in \nset$.

\begin{proposition}[Induced spatial process]
  \label{prop:spatial_process}
  For any $\bfx \in \calX^n$ with $n \in \nset$, $M \in \nset$ and $t \in \ccint{0,T}$:
  \begin{align*}
      \textstyle \bfY_{\x,t}^M = \textstyle \rme^{-t/2} \sum_{m=0}^M \lambda_m^{1/2}  Z_{m,0} e_m(\x)
      + \sum_{m=0}^M \lambda_m^{1/2} \int_0^t \rme^{-s/2} \rmd \bfB_s^m e_m(\bfx).
  \end{align*}
     In addition, let $\eta^M \in \mathcal{P}(\rmc(\coint{0,+\infty}, \rset)^{\mathcal{X}})$ be the distribution of $\bfY^M$. Then $\lim_{M \to + \infty} \eta^M = \eta$, where $\eta \in \mathcal{P}(\rmc(\coint{0,+\infty}, \rset)^{\mathcal{X}})$ is the distribution of $\bfY$ satisfying
    \begin{equation} \label{eq:multivariate_ou}
        \rmd \bfY_{\x,t} = -\tfrac{1}{2}  \bfY_{\x,t}~ \rmd t +  {\rmK}(\x,\x)^{1/2} \rmd \bfB_t  \eqsp .
    \end{equation}
\end{proposition}

See Appendix \ref{sec:spatial_process} for a proof of \Cref{prop:spatial_process}. This result justifies choosing a kernel to encode spatial correlation in the data since the induced spatial SDE correlation is (in the limit of an infinite number of spectral coefficients) given by this kernel matrix. Note that \Cref{prop:spatial_process} is given for the forward process but a similar derivation holds for the backward stochastic process. The time-reversal of \cref{eq:multivariate_ou} is given by $\rmd \overleftarrow{\bfY}_{\x,t} = \{\tfrac{1}{2}  \overleftarrow{\bfY}_{\x,t} + \rmK(\x,\x) \nabla \log p_{\x,T-t}(\overleftarrow{\bfY}_{\x,t}) \}~ \rmd t + {\rmK}(\x,\x)^{1/2} \rmd \bfB_t$, where $p_{\x,t}$ is the density of $\bfY_{\x,t}$.

When directly modelling \eqref{eq:multivariate_ou} (and its time-reversal), it is not clear whether the parameterised distributions satisfy the consistency criterion in the Kolmogorov extension theorem. This constraint is immediately satisfied in our approach. Additionally, at training time, we bypass any numerical instability or computational cost stemming from kernel matrix operations. Indeed, in our approach all costly kernel operations are incurred as a pre-processing step.

\subsection{Computing the Spectral Dataset}
\label{sec:compute_bases}

Our diffusion model is trained on a spectral dataset $\{\bfZ^s\}_{s=1}^S$, where $\bfZ=(Z_0, ..., Z_M) \in \mathbb{R}^{M+1}$ are the coefficients of the representation
\begin{equation}
    \textstyle \bfY_x^M = \sum_{m=0}^M \lambda_m^{\frac{1}{2}} Z_m e_m(x),
\end{equation}
as established in \Cref{sec:mercer}. We now show how the eigensystem $(\lambda_m, e_m)_{m \in \mathbb{N}}$, satisfying  $[\rmT_\rmK e_m](x) = \lambda_m e_m(x)$ for all $x \in \mathcal{X}$, can be obtained for a given continuous, symmetric and positive definite kernel $\rmK(x, x')$, and how the resulting spectral dataset is computed.

We begin by assuming a dataset of stochastic process realisations $\mathcal{D} = \{\{\bfY^s(x_i^s)\}_{i=1}^{n^s}\}_{s=1}^S$ where each process is evaluated on a finite subset of index points $\Xi^s := \{x_i^s\}_{i=1}^{n^s} \subset \mathcal{X}$. In most instances of practical interest, the dataset naturally gives rise to a common evaluation set $\Xi = \bigcap_{s=1}^S\Xi^s$, for instance time series data is often recorded at fixed intervals. Given a non-empty $\Xi$, we use the Nystr\"{o}m method~\citep{Atkinson_1997} to approximate the eigensystem of kernel $\rmK$. This involves forming the Gram matrix of the kernel over the shared evaluation set, $\tilde{\rmK} = (\mathrm{K}_{ij})_{i,j=1}^n \in \mathbb{R}^{n\times n}$ where $\rmK_{ij} = \rmK(x_i, x_j)$ for $x_i, x_j \in \Xi$ and $n = |\Xi|$ denotes the cardinality of the shared evaluation set. After finding the eigenvalues $\lambda_m$ and eigenvectors $u_m$ of the (normalised) Gram matrix $\frac{1}{n}\tilde{\rmK}$, the Nystr\"{o}m method gives an approximation of the continuous eigenfunctions $e_m(x)$ as
\begin{equation}
    \textstyle \hat{e}_m(x) =  (\lambda_m \sqrt{n})^{-1} \sum_{i=1}^n\rmK(x, x_i) u_m^{(i)}
\end{equation}
where $u_m^{(i)}$ denotes the $i^\text{th}$ element of $u_m$. A full derivation is given in Appendix~\ref{app:compute_eigenfunctions}.

In some cases $\Xi$ may be empty or $n = |\Xi|$ too small for a reliable Nystr\"{o}m approximation. The Nystr\"{o}m method relies on an $n$-point quadrature and converges to the true eigensystem only as $n \to \infty$ \citep{baker1979Numerical}. At small $n$, it only resolves the leading low frequency modes accurately. For instance, for a stationary kernel on a uniform grid over $[0,L]$, modes above the Nyquist frequency $\frac{n}{2L}$ cannot be recovered. With small $n$ we therefore may not be able to recover a sufficient fraction of the total spectral energy using our representation. In this case, we can instead choose a known kernel with an analytic eigenbasis \citep[Section 4.3]{rasmussen2006Gaussian}, for example the RBF kernel admits analytic eigenfunctions under a Gaussian input measure. On the other hand, if $\Xi = \Xi^s \; \forall s$ and we are only interested in modelling the function at the evaluation set $\Xi$, then we can decompose the empirical covariance matrix directly and drop the Nystr\"{o}m off-grid extension, which brings the benefits of the Karhunen-Lo\`{e}ve theorem discussed in \Cref{sec:mercer,sec:method_properties}.

Finally we compute the spectral dataset. Given stochastic process dataset $\mathcal{D} = \{\{\bfY^s(x_i^s)\}_{i=1}^{n^s}\}_{s=1}^S$ and approximate eigensystem $(\hat{\lambda}_m, \hat{e}_m)$ as above, we fix a truncation order $M< |\calX|$ such that we capture $99\%$ of the energy of the eigenspectrum and compute a new $M+1$-dimensional dataset $\calD^M = \{\bfZ^s\}_{s=1}^S = \left\{\{Z_m^s\}_{m=0}^{M}\right\}_{s=1}^S$, where $Z_m^s = \lambda_m^{-1/2} \langle \bfY^s, e_m \rangle \approx \frac{1}{n^s} \sum_{i=1}^{n^s} \hat{\lambda}_m^{-1/2} \bfY^s(x_i^s) \hat{e}_m(x_i^s)$. We remark that approximating the eigenbasis and forming the spectral dataset may appear costly, however it is strictly a one-off `pre-processing' cost independent of training and sampling. We provide the cost analysis in Appendix~\ref{app:compute_eigenfunctions}. We also include a complete algorithm for our approach in Appendix~\ref{app:algorithm}.

\section{Related Work}

\label{sec:related_work}

\paragraph{Gaussian processes and the neural processes family.} One standard and powerful framework to construct distributions over functional spaces is Gaussian processes \citep{rasmussen2003gaussian}. However, they are restricted in their modelling capacity, and exact GPs scale badly with the number of datapoints. These problems are partially alleviated by using neural processes~\citep{kim2019Attentive,garnelo2018neural,garnelo2018Conditional,jha2022neural,louizos2019functional,singh2019sequential,markou2022Practical}, although they also assume a Gaussian likelihood and only satisfy predictive consistency---meaning one must condition on a context set to achieve consistency.

\paragraph{Spatial structure in diffusion models.} A variety of approaches have also been proposed to incorporate spatial correlation in the noising process of finite-dimensional diffusion models leveraging the multi-scale structure of data \citep{jing2022subspace,guth2022wavelet,ho2022cascaded,saharia2021image,hoogeboom2022blurring,rissanen2022Generative}. Our methodology can also be seen as a principled way to modify the forward dynamics in classical denoising diffusion models. Indeed by applying the diffusion in the spectral space the destruction process \emph{blurs} the samples instead of \emph{noising} them, see \Cref{fig:mnist_trajectories} for an illustration. Therefore, our contribution can be understood in the light of recent advances in generative modelling on cold and soft denoising diffusion models \citep{daras2022soft,bansal2022cold,hoogeboom2022blurring}. Several recent works explicitly introduced a covariance matrix in the Gaussian noise, either based on a choice of kernel \citep{bilos2022Modeling}, on the Discrete Fourier Transform of images \citep{voleti2022Scorebased}, or via empirical second order statistics for protein modelling \citep{ingrahamIlluminating}.

\paragraph{Infinite dimensional diffusion models.} Infinite dimensional diffusion models have also been studied in \cite{kerrigan2022Diffusion,pidstrigach2023infinite,lim2023score,bond2023infty,hagemann2023multilevel,franzese2023continuous,dutordoir2022neural}. These works extend classical diffusion models \cite{song2020score,ho2020denoising} to the infinite-dimensional space, using techniques from the Cameron-Martin theory to define infinite-dimensional Gaussian measures and infinite dimensional SDEs. These models are equivalent to ours when projected in a spectral space, see Appendix~\ref{sec:spatial_process}. In practice, however, such functional models have to be discretised for training and sampling purposes, and consequently lose consistency over the finite marginals.

\paragraph{Spectral approaches.} One possibility to circumvent the consistency problem is to consider spectral decompositions---such as the Karhunen-Lo\`{e}ve one---which shows that stochastic processes can be represented over a functional basis with random coefficients.  Some prior works have explored this avenue using variational auto-encoders to model these coefficients~\citep{mishra2020pi} whereas \citet{lim2022mathcal} considered energy-based models.  Another related line of work is `inter-domain Gaussian processes'~\citep{gredilla2009Interdomain}, which rely on `inducing' variables that are obtained via inducing functions such as elements of a Fourier basis~\citep{hensman2018Variational}.

\begin{figure*}[t!]
    \centering%
    \raisebox{-0.5\height}{\includegraphics[width=0.76\textwidth]{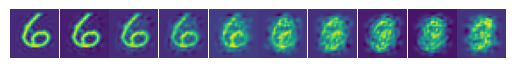}}
    \hfill
    \raisebox{-0.5\height}{\includegraphics[width=0.10\textwidth]{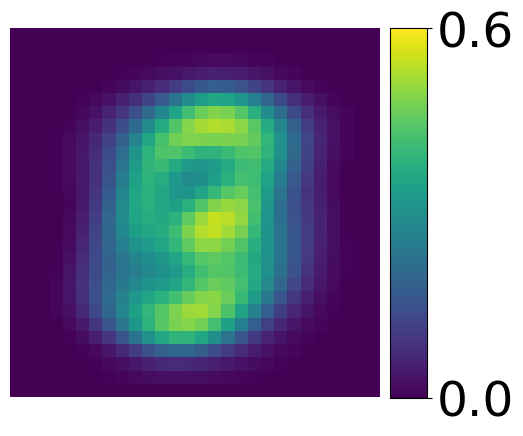}} 
    \hfill
    \raisebox{-0.5\height}{\includegraphics[width=0.10\textwidth]{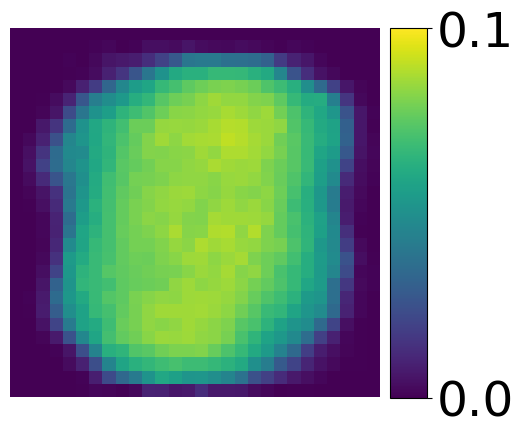}}\\[-0.8em]%
    \raisebox{-0.5\height}{\includegraphics[width=0.76\textwidth]{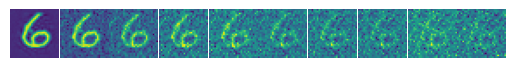}}
    \hfill
    \raisebox{-0.5\height}{\includegraphics[width=0.10\textwidth]{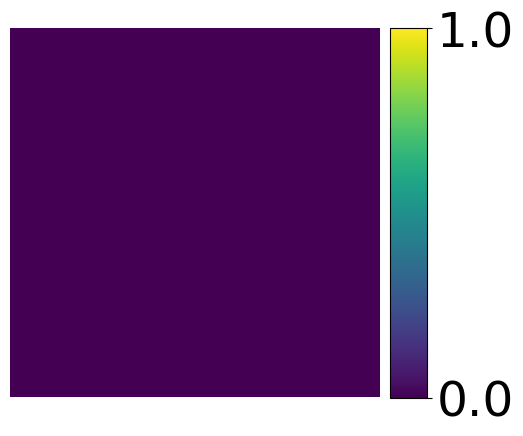}}
    \hfill
    \raisebox{-0.5\height}{\includegraphics[width=0.10\textwidth]{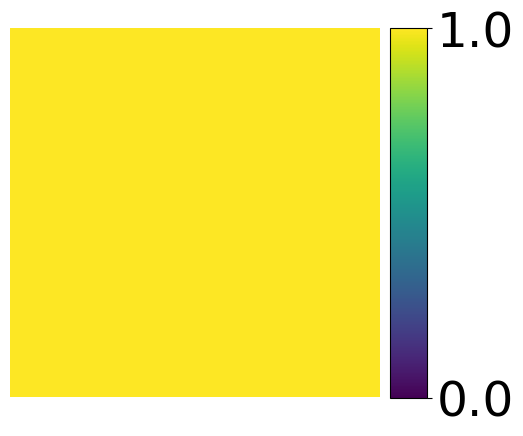}}\\
    \vspace{-0.8em}
    \captionsetup{font=footnotesize}
    \caption{Induced \emph{spatial} forward process $\bfY_{\bfx, t}^M$ of our model using the covariance kernel [Top] vs standard OU process [Bottom] on an MNIST sample. Pixel-wise mean and standard deviation of the reference measure are reported in rightmost columns respectively. Our induced diffusion process follows a correlated OU process in the spatial domain, converging to the Gaussian process closest in KL divergence to the data distribution.}
    \label{fig:mnist_trajectories}
\vspace{-1.5em}
\end{figure*}

\section{Experimental Results}
\label{sec:experiments}

\begin{wraptable}[15]{r}{0.3\textwidth}
    \vspace{-1.3em}
    \setlength{\tabcolsep}{0pt}
    \begin{tabular}{c} 
         \includegraphics[width=1.\linewidth]{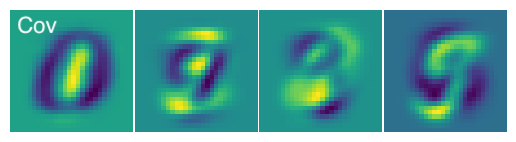}\\ [-1em]
         \includegraphics[width=1.\linewidth]{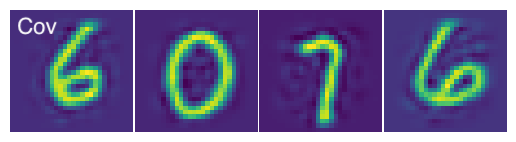}\vspace{-.5em}\\
        \includegraphics[width=1.\linewidth]{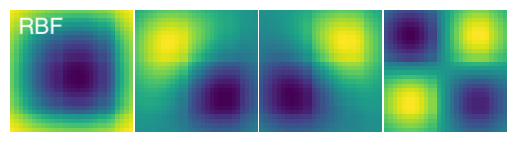}\\ [-1em]
        \includegraphics[width=1.\linewidth]{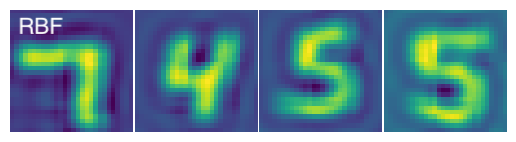}
    \end{tabular}
    \vspace{-.2em}
    \captionof{figure}{Eigenfunctions [Top] and samples from our method [Bottom] for covariance (Cov) and RBF kernels.}
    \label{fig:kernel_examples}
\end{wraptable}

We first illustrate our methodology by synthesising images from the MNIST dataset. While this is not a true functional modelling task, it allows for an informative visual illustration of the properties of our method. Next, we model real datasets which exhibit challenging non-stationarity and bi-modality. Finally, we assess our method's ability to model predictive distributions by conditioning on a set of observations. We compare against Gaussian processes (GPs), latent neural processes (NPs) and variants, and neural diffusion processes (NDPs) \citep{dutordoir2022neural}, a recent diffusion-based alternative. Experimental details can be found in Appendix~\ref{sec:experimental-details}.

\paragraph{Modelling 2D images.}

We use the MNIST image dataset to illustrate the key components of our proposed method (see \Cref{sec:mercer,sec:model}). First, we investigate the influence of the choice of the kernel. In \Cref{fig:kernel_examples}, we present numerical estimates of the four eigenfunctions with the highest spectral energy, alongside the associated samples from the model. Qualitatively, eigenfunctions of the covariance kernel capture spatial correlations effectively and thus generate the best samples, while the Radial Basis Function (RBF) kernel imposes a smoothness which tends to produce blurry samples.

\begin{figure}[h]
    \centering
    \begin{subfigure}[b]{0.45\textwidth}
    \includegraphics[width=\linewidth]{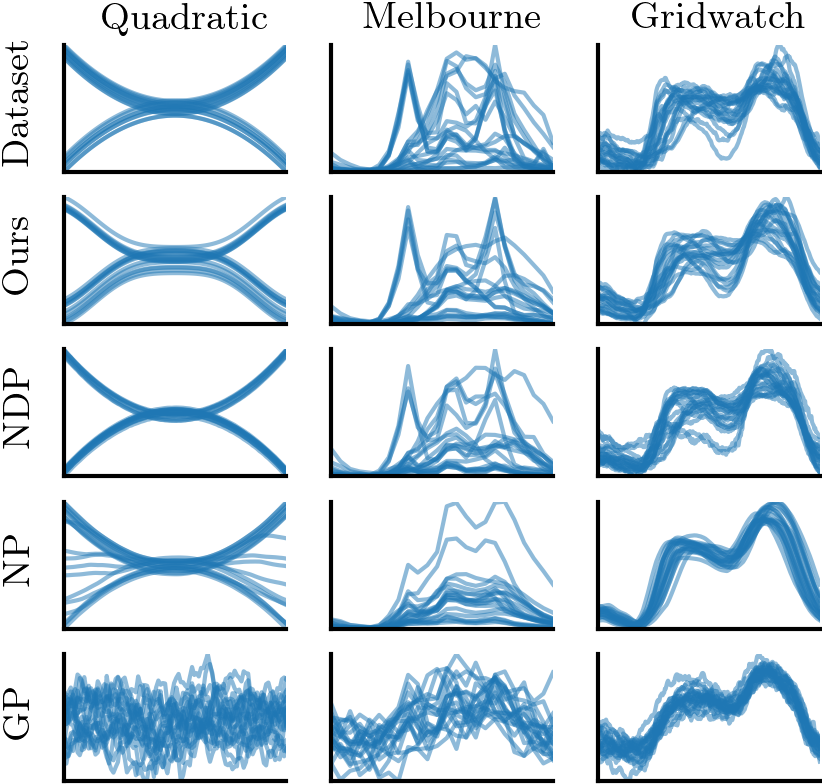}
    \caption{
    Prior samples.
    }
    \label{fig:unconditional}
    \end{subfigure}
     \hfill
    \begin{subfigure}[b]{0.45\textwidth}
    \includegraphics[width=0.9\linewidth]{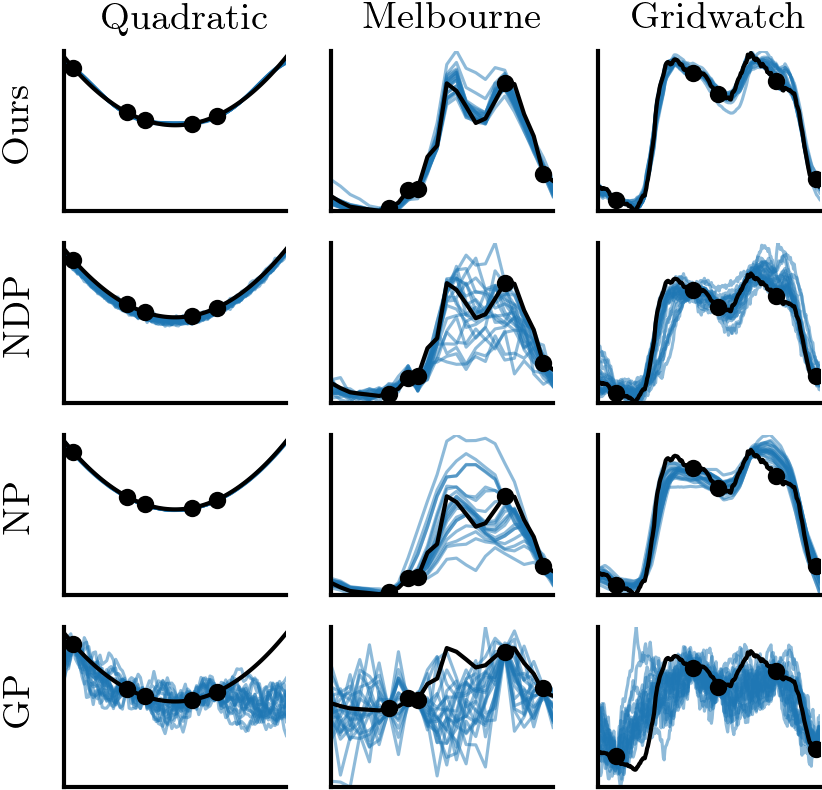}
    \caption{
    Posterior samples in blue, true samples in black, context points in black.
    }
    \label{fig:1d_conditional}
    \end{subfigure}
    \caption{
    Samples from our model alongside trained NDPs, NPs and GPs on the different 1D datasets.
    }
\end{figure}

\paragraph{Modelling 1D stochastic processes.}
\label{sec:modelling_1d_sps}

Next, we assess the probabilistic modelling capacity of our proposed methodology, and in particular on processes with multi-modal marginal densities. To do so we train our model, GPs, NPs and NDPs \citep{dutordoir2022neural} on several 1D datasets. We consider the \textsc{Melbourne} dataset, which records the number of pedestrians every hour on $10$ different streets of Melbourne, and the \textsc{Gridwatch} dataset where each sample corresponds to the energy demand on the UK National Grid. Additionally, we generate a synthetic bimodal \textsc{Quadratic} dataset as \(f(x) := a x^2 + \epsilon\) with \(a \sim \mathrm{U}(\{-1,1\})\), \(\epsilon \sim \mathrm{N}(0,10)\) and $x \sim \mathrm{U}([-10,10])$. \Cref{fig:unconditional} shows samples from these datasets, along with samples generated by trained NDPs, NPs, GPs and our model. We clearly observe that our method and NDPs are capable of expressing and fitting bi-modality on \textsc{Quadratic}, whilst both NPs and GPs assume a Gaussian likelihood, which prevents them from correctly fitting the bimodal marginals. Additionally, we see that NPs underestimate the variance on \textsc{Melbourne} and \textsc{Gridwatch}. To provide a quantitative assessment of these observations, we compute the power of a two-sample hypothesis test for functional data \citep{wynne2022kernel} which aims at discriminating between samples from the model and the dataset. \Cref{tab:quant_results} shows that our method consistently outperforms GPs, NPs and NDPs, confirming the qualitative observations.

\looseness=-1 We also evaluate the predictive performance of our method on these 1D datasets. We use our model and the baselines to predict the entire time series based on a context set of observations uniformly sampled across the input domain. We consider context sets of between $5$ and $50$ observations ($5$ and $12$ on \textsc{Melbourne}) and report the MSE between the predictions and the true sample in \Cref{tab:conditional_1d}. We see that our method outperforms GPs on all datasets and NDPs and NPs on both real datasets. We also visualise predictions on $5$ context points in \Cref{fig:1d_conditional}. Here it is clear that our model is able to capture the context points and correctly interpolate between contexts. NDPs produce the next most convincing samples, although do not appear to match the context points quite as closely as our method. NPs do not appear to accurately fit the context points and GPs fail to capture the length-scale of the data.

\begin{minipage}{\textwidth}
\begin{minipage}[t]{\textwidth}
    \centering
    \small
    \setlength{\tabcolsep}{4pt}
    \begin{tabular}{c | c c c c}
        \toprule
         & Ours & NDP & NP & GP \\
        \midrule
        Quadratic & \(\bf{{5.4}_{\pm 1.6}}\) & \(\bf{{5.5}_{\pm 0.6}}\) & \(\bf{{7.0}_{\pm 1.4}}
\) & \({100.0}_{\pm 0.0}\) \\
        Melbourne & \(\bf{{5.3}_{\pm 0.7}}\) & \({6.4}_{\pm 0.2}\) & \({10.1}_{\pm 1.9}\) & \({20.1}_{\pm 4.0}\) \\
        Gridwatch & \(\bf{{4.7}_{\pm 0.5}}\) & \({6.9}_{\pm 0.8}\) & \({51.8}_{\pm 15.1}\) & \({29.2}_{\pm 5.5}\) \\
        \bottomrule
    \end{tabular}
    \vspace{1em}
    \captionof{table}{Power ($\%$) of a kernel two-sample hypothesis test on 1D datasets. Lower is better. Statistically significant best result is in \textbf{bold}.}
     \label{tab:quant_results}
    \end{minipage}

    \begin{minipage}[t]{\textwidth}
     \centering
     \small
     \setlength{\tabcolsep}{4pt}
        \begin{tabular}{c|cccc}
        \toprule
         & Ours & NDP & NP & GP \\
        \midrule
        Quadratic & $\bf{0.033_{\pm 0.003}}$ & \(\bf{{0.039}_{\pm 0.01}}\) & $\bf{0.030_{\pm0.02}}$ & $1.6_{\pm2.0}$ \\
        Melbourne & $\bf{0.23_{\pm0.01}}$ & \({1.3}_{\pm 0.8}\) & $2.0_{\pm0.1}$ & $11_{\pm2}$ \\
        Gridwatch & $\bf{0.034_{\pm0.02}}$ & \({0.17}_{\pm 0.1}\) & $0.70_{\pm0.01}$ & $9.6_{\pm0.6}$\\
        \bottomrule
    \end{tabular}
    \vspace{1em}
    \captionof{table}{Predictive MSE ($10^{-1}$) on 1D datasets. Number of conditioning points is averaged uniformly over $[5,50]$ ($[5,12]$ for Melbourne).}
    \label{tab:conditional_1d}
    \end{minipage}
\end{minipage}

\paragraph{1D Gaussian process regression.}
\label{sec:1d_gp_regression}
 
While the prior experiments on real datasets revealed the ability of our method to condition on context points and achieve competitive predictive MSE, we were not able to accurately assess the quality of the learned conditional distribution beyond its mean. In this section we fit models to Gaussian process posteriors where the covariance of the conditional process is analytically available for evaluation of the learned distributions.

The number of context observations $|C|$ is uniformly sampled between 1 and 50, while the locations $\{x^i\}_{i \in C}$ are uniformly sampled over $[-2, 2]$. The evaluation locations are similarly sampled between $[-2, 2]$, although their number is kept fixed to $100$. We additionally train NDPs along with three variants of neural processes (NPs), latent NPs \citep{garnelo2018neural}, attentive NPs (ANPs) \citep{kim2019Attentive} which use a cross-attention module to produce a query-specific context embedding, and attentive Gaussian NPs (AGNPs) \citep{markou2022Practical} which additionally parameterise the full likelihood covariance. In \Cref{tab:predictive_error} we report the predictive mean and covariance errors for each of the benchmarks, alongside our model. We include variants of our method which use the analytic RBF and Mat\'{e}rn kernels and also the empirically estimated covariance kernel (Cov). Additionally in \Cref{fig:conditional_gp_dataset} we show samples from the learnt predictive distributions, alongside the target GP mean and credible region. We observe that our approach yields the best predictive modelling performance both quantitatively and qualitatively. We also note that our method performs better than the NP variants across all choices of kernel in the functional decomposition. The AGNP is the most competitive NP variant on predictive covariance, but struggles on the predictive mean. A clear advantage of our spectral representation \eqref{eq:parameterization} is that the choice of functional basis naturally induces covariance structure onto the modelled stochastic process, even when the spectral coefficients are uncorrelated. In contrast, NPs assume a Gaussian likelihood with diagonal covariance $\mathrm{K}_{ij} = \sigma(x_i, r)^2 \updelta_{ij}$ with $r$ being the stochastic latent embedding of the context set, whilst AGNPs parameterise a full covariance matrix as $\mathrm{K}_{ij} = g(x_i, r)^\top g(x_j, r)$ with $g$ a neural network. 
\begin{wrapfigure}[27]{r}{.4\textwidth}
    \centering
    \includegraphics[width=\linewidth]{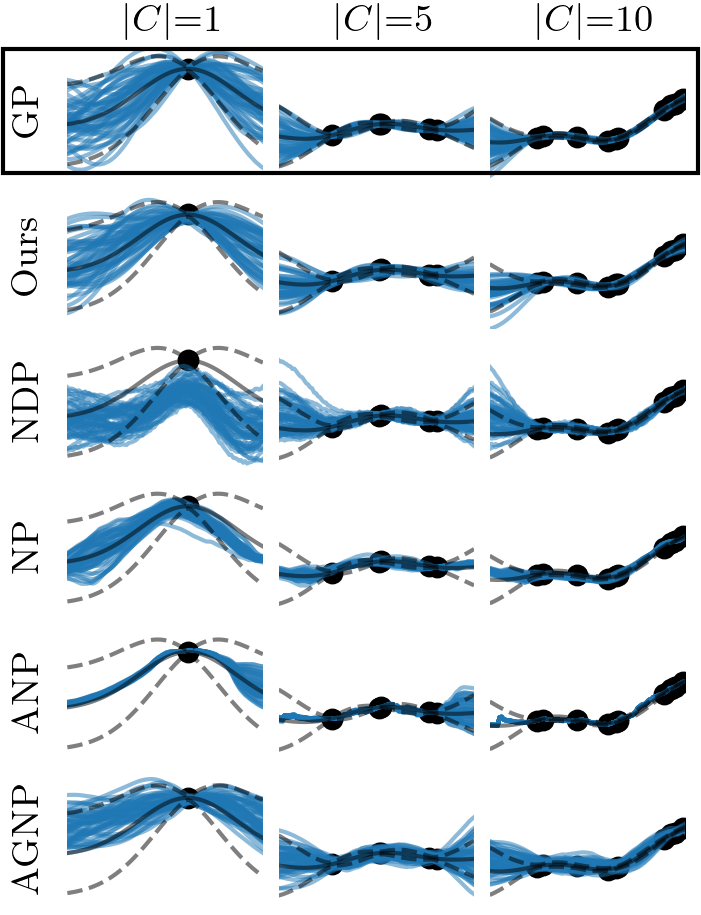}
    \vspace{-.8em}
    \caption{Posterior samples (paths) are shown in thin black for each model, while the true GP posterior mean and standard deviation are shown respectively in thick solid and dashed black. The black dots are context points being conditioned on, and their number is increasing from left to right.
    }
    \label{fig:conditional_gp_dataset}
 \end{wrapfigure}
 As such, these methods are required to learn the covariance structure entirely from the data. Indeed, in \Cref{fig:covariance}, we see that AGNPs have the most expressive covariance structure amongst the NP variants, but our approach most accurately matches the true predictive covariance.

\begin{figure}[t]
\centering
\begin{minipage}[t]{\textwidth}
\begin{minipage}[t]{0.5\textwidth}
\vspace{.0em}
        \centering
        \footnotesize
        \setlength{\tabcolsep}{4pt}
            \begin{tabular}{c|cccc}
                \toprule
                 & \multicolumn{2}{c}{RBF} & \multicolumn{2}{c}{Mat\'ern} \\ 
                 Models & Mean & Cov & Mean & Cov \\
                 \midrule
                Ours (RBF) & \({13}_{\pm 0}\) & \(\bf{{0.3}_{\pm 0.2}}\) & \({40}_{\pm 2}\) & \(\bf{{0.5}_{\pm 0.2}}\) \\
                Ours (Mat\'ern) & \(\bf{{4}_{\pm 0}}\) & \(\bf{{0.2}_{\pm 0.1}}\) & \({15}_{\pm 6}\) & \(\bf{{0.4}_{\pm 0.2}}\) \\
                Ours (Cov) & \({13}_{\pm 1}\) & \(\bf{{0.4}_{\pm 0.2}}\) & \(\bf{{7}_{\pm 1}}\) & \(\bf{{0.5}_{\pm 0.4}}\) \\
                NDP & \({33}_{\pm 4}\) & \({2.9}_{\pm 1.8}\) & \({35}_{\pm 4}\) & \({3.2}_{\pm 1.9}\) \\
                NP & \({102}_{\pm 8}\) & \({4.5}_{\pm 1.2}\) & \({63}_{\pm 7}\) & \({4.3}_{\pm 1}\) \\
                ANP & \({76}_{\pm 14}\) & \({4.4}_{\pm 1.5}\) & \({41}_{\pm 19}\) & \({6.8}_{\pm 1.4}\) \\
                AGNP & \({107}_{\pm 17}\) & \({1.7}_{\pm 0.1}\) & \({44}_{\pm 2}\) & \({2}_{\pm 0.4}\) \\
                \bottomrule
            \end{tabular}
        \captionof{table}{
        Error \((10^{-3})\) over predictive mean and covariance on GP datasets, resp.\ measured with $\ell_2$ and Frobenius distances.
        Statistically significant best results are in \textbf{bold}.
        }
        \label{tab:predictive_error}
    \end{minipage}
    \hfill
    \begin{minipage}[t]{0.47\textwidth}
    \vspace{-.3em}
        \centering
        \includegraphics[width=\linewidth]{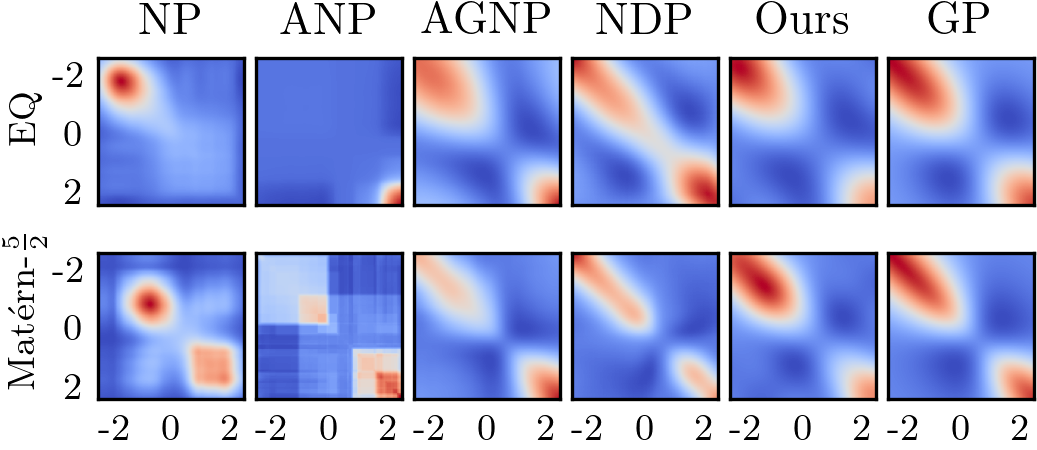}
        \captionof{figure}{
        Predictive covariance matrices for each method and true covariances, for Gaussian process with RBF [Top] or Mat\'ern [Bottom] kernels, conditioning on 1 observation.
        }
        \label{fig:covariance}
    \end{minipage}
\end{minipage}
\vspace{-1em}
\end{figure}

\section{Discussion}

In this work we introduced a diffusion-based model over stochastic processes. Our approach constructs a finite representation of an uncountable process by decomposing in the spectral space, where we truncate the representation and build a finite-dimensional diffusion model over the stochastic spectral coefficients. This choice defines a model which directly outputs valid stochastic processes, thereby satisfying consistency and exchangeability requirements for free while enjoying the full flexibility of diffusion generative models. We empirically demonstrated that our method shows great capacity for modelling complex functional distributions. Simple amortisation of the score over context sets also yields state-of-the-art predictive posteriors.

In this work, we relied on amortisation of the context observations for conditional modelling, but we believe that particle filtering approaches \citep{trippe2022Diffusion} would be promising in this setting. Furthermore, we leave the exploration of other functional bases, such as wavelets for computer vision applications, to future work.

\newpage
\section*{Acknowledgements}

AP and MH are funded through the StatML CDT through grant EP/S023151/1. AD acknowledges support of the UK Defence Science and Technology Laboratory (Dstl) and Engineering and Physical Research Council (EPSRC) under grant EP/R013616/1. This is part of the collaboration between US DOD, UK MOD and UK EPSRC under the Multidisciplinary University Research Initiative. AD is also partially supported by the EPSRC grant EP/R034710/1 CoSines.
EM is supported by an EPSRC Prosperity Partnership EP/T005386/1 between Microsoft Research and the University of Cambridge.

\bibliography{bibliography.bib}

\newpage
\onecolumn
\appendix
\appendixhead

\section{Methodological Details of Spectral Diffusion Models}
\label{app:methodology}

In this section, we provide additional methodological details beyond those presented in \Cref{sec:method}.

\subsection{Extension to $\mathbb{R}^d$-Valued Processes}
\label{app:rd_processes}

Our presentation in \Cref{sec:method} was specialised to 1-dimensional stochastic processes for simplicity. We therefore begin by extending the presentation to general $\rset^d$-valued stochastic processes. In this case we have $\bfY_x \in \rset^d$, the kernel is matrix-valued $\rmK: \calX \times \calX \to \rset^{d \times d}$ and eigenfunctions are vector-valued $e_m: \calX \to \mathbb{R}^d$. We present the remainder of this section using the $\rset^d$ convention unless stated otherwise, to illustrate how the methodological details extend to this setting. 

We also note that further consideration of the choice of kernel $\rmK: \calX \times \calX \to \rset^{d \times d}$ is required in the $\rset^d$ case. For instance, one could choose a diagonal structure which models each output dimension independently. This choice affects the GP represented by the base distribution (see \Cref{prop:proj_gaussian_process}), but does not prevent the model from capturing correlations over output dimensions: the diffusion model learns a full joint distribution over the coefficients $\bfZ\in\rset^{d(M+1)}$ which still induces correlations over output dimensions in the reconstruction $\bfY^M$. Choosing such a diagonal structure reduces the computational cost of computing the eigenbasis, see Appendix~\ref{app:compute_eigenfunctions}.

\subsection{Karhunen-Lo\`{e}ve Theorem}
\label{app:KL_theorem}

Here we present the Karhunen-Lo\`{e}ve Theorem and derive the independence of the spectral coefficients when $\rmK = \rmK_\bfY$. 

\begin{theorem}[Karhunen--Lo\`{e}ve~\citep{loeve1978probability}]
    \label{thm:KL_theorem}
    Let $\calX$ be a compact space and $(\bfY_x)_{x \in \calX}$ a continuous stochastic process with $\bfY_x \in \rset^d$ such that for any $x \in \calX$, $\expeLigne{\bfY_x} = 0$ and $\expeLigne{\normLigne{\bfY_x}^2} < + \infty$. For any $x_1, x_2 \in \calX$, we denote the covariance kernel as $\rmK_\bfY(x_1,x_2) = \expeLigne{\bfY_{x_1} \bfY_{x_2}^T}$ and define an $\mathrm{L}^2(\calX, \rset^d)$ operator $\rmT_{\rmK_\bfY}$ given for any $f \in \mathrm{L}^2(\calX, \rset^d)$ by
    \begin{equation} 
        \label{eq:app_operator_kernel_definition}
        \textstyle \rmT_{\rmK_\bfY} [f](\cdot) = \int_\calX \rmK_\bfY(x,\cdot) f(x) \rmd x.
    \end{equation}
    Denoting by $(e_m)_{m \in \nset} \in (\mathrm{L}^2(\calX, \rset^d))^\nset$ and $(\lambda_m)_{m \in \nset} \in \rset^\nset$ the eigenfunctions and eigenvalues of $\rmT_{\rmK_\bfY}$, and for any $m \in \nset$ setting $\textstyle{Z_m = \lambda_m^{-1/2} \int_\calX  \bfY_x^T e_m(x)  \rmd x}$, we have
    \begin{equation}
        \textstyle \lim\limits_{M \to +\infty} \mathbb{E}[\sup_{x \in \calX} \|\bfY_x - \sum_{m=0}^M {\lambda_m}^{1/2} Z_m e_m(x)\|^2] = 0.
    \end{equation}
\end{theorem}

We now show that the random coefficients from the Karhunen-Lo\`{e}ve basis expansion, given for any $m \in \nset$ by
\begin{equation}
  \textstyle{
    Z_m = \langle \bfY, \lambda_m^{-1/2} e_m\rangle_{\rmL^2(\mathcal{X})} = \lambda_m^{-1/2} \int_\mathcal{X} \bfY(x)^T e_m(x) \rmd x \eqsp ,
    }
\end{equation}
satisfy \(\mathbb{E}[Z_m] = 0\) and \(\mathbb{E}[Z_{m}Z_{m'}] = \updelta_{m, m'}\).  First, we have that for any $m \in \nset$,
\begin{equation}
  \textstyle \mathbb{E}[Z_m] = \lambda_m^{-1/2} \int_\mathcal{X} \mathbb{E}[\bfY(x)]^T  e_m(x)\rmd x = 0.
\end{equation}
Then, we have for any $m, m' \in \nset$:
\begin{align}
    \mathbb{E}[Z_mZ_{m'}] &= \textstyle{ \mathbb{E}\left[\lambda_m^{-1/2}\lambda_{m'}^{-1/2} \int_\mathcal{X}  \bfY(x)^T e_m(x) \rmd x \int_\mathcal{X} \bfY(x')^T e_{m'}(x')\rmd x'\right]} \\
    &=  \mathbb{E}\left[\lambda_m^{-1/2}\lambda_{m'}^{-1/2} \textstyle{ \int_\mathcal{X}\int_{\mathcal{X}} (\bfY_x^Te_m(x)) (\bfY_{x'}^Te_{m'}(x')) \rmd x \rmd x'}\right] \\
    &=  \lambda_m^{-1/2}\lambda_{m'}^{-1/2} \textstyle{ \int_\mathcal{X}\int_{\mathcal{X}} e_m(x)^T\mathbb{E}\left[\bfY_x\bfY_{x'}^T\right]e_{m'}(x') \rmd x \rmd x'}\\
    &=  \lambda_m^{-1/2}\lambda_{m'}^{-1/2} \textstyle{ \int_\mathcal{X}\int_{\mathcal{X}} e_m(x)^T \rmK_\bfY(x, x')e_{m'}(x') \rmd x \rmd x'}\\
    &= \lambda_m^{-1/2} \lambda_{m'}^{-1/2} \textstyle{ \int_\mathcal{X} \lambda_m e_m(x')^T e_{m'}(x')  \rmd x'}= \lambda_m^{-1/2} \lambda_{m'}^{-1/2} \lambda_m \updelta_{m, m'}= \updelta_{m, m'} \eqsp . 
\end{align}

\subsection{Consistency and Exchangeability}
\label{app:consistency_exchangability_proof}

As discussed in \Cref{sec:mercer}, our representation \Cref{eq:parameterization} is by construction a stochastic process. Our model therefore satisfies consistency and exchangeability by invoking the trivial direction of the Kolmogorov extension theorem (\Cref{thm:ket}). For illustrative purposes, we also provide a direct proof of this result below. The proof hinges on the conditional independence of evaluations of $\bfY^M$ given a realisation of the stochastic coefficients $\bfZ$.

\begin{proposition}
  \label{prop:exchangeability_consistency_app}
 We denote by $\pi^M$ the distribution of $\bfY^M = (\sum_{m=0}^M {\lambda_m}^{1/2} Z_m e_m(x))_{x \in \calX}$ and $\rmS = \ensembleLigne{\pi^M_{x_1, \dots, x_n}}{(x_1, \dots, x_n) \in \calX^n,
  \ n \in \nset}$.
  Then we have that $\rmS$ is \emph{exchangeable} and \emph{consistent}.
\end{proposition}

\begin{proof}{}{}
 We denote $\bfZ = \{Z_m\}_{m=0}^M$ and $\pi$ the model
  distribution on $\bfZ$.  For any $n \in \nset$ and
  $f \in \rmc_b((\rset^d)^n, \rset)$ we have
\begin{align}
\textstyle{\int f(y_1, \dots, y_n) \rmd \pi^M_{x_1, \dots, x_n}(y_1, \dots, y_n) }
&=\textstyle{ \int \int_{\rset^{M+1}} f(y_1, \dots, y_n) \rmd \pi^M_{x_1, \dots, x_n}(y_1, \dots, y_n, \bfZ=z) \rmd z} \\ 
&=\textstyle{ \int \int_{\rset^{M+1}} f(y_1, \dots, y_n)\rmd \pi^M_{x_1, \dots, x_n}(y_1, \dots, y_n, | \bfZ=z) \rmd \pi (z)  } \\
&=\textstyle{ \int \int_{\rset^{M+1}} f(y_1, \dots, y_n) \prod_{i=1}^n \rmd \pi^M_{x_i}(y_i|\bfZ=z) \rmd \pi (z) \eqsp .}
\end{align}
Hence, for any $n \in \nset$, $n$-permutation $\sigma \in \mathfrak{S}_n$,
$x_1, \dots, x_n \in \calX$, and continuous and bounded function
$f \in \rmc_b((\rset^d)^n)$ we have
$$\textstyle{ \int f(y_{\sigma^{-1}(1)}, \dots, y_{\sigma^{-1}(n)}) ~\rmd \pi^M_{x_{\sigma(1)}, \dots, x_{\sigma(n)}}(y_1, \dots, y_n) 
= \int f(y_{1}, \dots, y_{n}) ~\rmd \pi^M_{x_1, \dots, x_n}(y_1, \dots, y_n) \eqsp .}$$
Similarly, for any $n_1 \leq n_2$,
$x_1, \dots, x_{n_2} \in \calX$ and $f \in \rmc_b((\rset^d)^{n_1})$ we have
\begin{align}
&\textstyle{\int f(y_{1}, \dots, y_{n_1})~\rmd \pi^M_{x_1, \dots, x_{n_2}}(y_1, \dots, y_{n_2})}
  \\
  & \qquad = \textstyle{ \int f(y_{1}, \dots, y_{n_1}) \int_{\rset^{M+1}} ~\rmd \pi^M_{x_1, \dots, x_{n_2}}(y_1, \dots, y_{n_2}|\bfZ=z)  \rmd \pi(z) } \\
& \qquad = \textstyle{ \int f(y_{1}, \dots, y_{n_1}) \int_{\rset^{M+1}} ~\prod_{i=1}^{n_2} \rmd \pi^M_{x_i}(y_i|\bfZ=z)  \rmd \pi(z) } \\
& \qquad = \textstyle{ \int f(y_{1}, \dots, y_{n_1}) \int_{\rset^{M+1}} ~\prod_{i=1}^{n_1} \rmd \pi^M_{x_i}(y_i|\bfZ=z) \left( \prod_{i=n_1+1}^{n_2} \rmd \pi^M_{x_i}(y_i|\bfZ=z) \right) \rmd \pi(z) } \\
& \qquad = \textstyle{ \int f(y_{1}, \dots, y_{n_1}) \int_{\rset^{M+1}} ~\prod_{i=1}^{n_1} \rmd \pi^M_{x_i}(y_i|\bfZ=z)  \rmd \pi(z)} \\
& \qquad = \textstyle{ \int f(y_{1}, \dots, y_{n_1})~\rmd \pi^M_{x_1, \dots, x_{n_1}}(y_1, \dots, y_{n_1}) \eqsp .}
\end{align}
In this setting, the random variable $\bfZ$ is finite dimensional as it is supported in $\rset^{M+1}$, but more generally this result is still true with an (infinite dimensional) process $F$.
\end{proof}

\subsection{Nystr\"{o}m Approximation of Eigenfunctions and the Spectral Dataset}
\label{app:compute_eigenfunctions}

Here we expand on the approximation of the eigenbasis and computation of the spectral dataset outlined in~\Cref{sec:compute_bases}. We begin by recalling eigensystem $(\lambda_m, e_m)_{m \in \mathbb{N}}, \lambda_m \in \rset, e_m \in \rmL^2(\calX, \rset^d)$ satisfying $[\rmT_\rmK e_m](x) = \lambda_m e_m(x)$ for all $x \in \mathcal{X}$, with $\rmT_\rmK$ the $\rmL^2(\calX, \rset^d)$ operator
$$
    \textstyle [\rmT_\rmK f](\cdot) = \int_\calX \rmK(\cdot, x) f(x) \rmd x.
$$
We recall also the stochastic process dataset $\mathcal{D} = \{\{\bfY^s(x_i^s)\}_{i=1}^{n^s}\}_{s=1}^S$ where each process is evaluated on a finite subset of index points $\Xi^s := \{x_i^s\}_{i=1}^{n^s} \subset \mathcal{X}$, and we form the shared evaluation set $\Xi = \bigcap_{s=1}^S\Xi^s$ which we assume here to be non-empty and of sufficient size to resolve the required modes of the basis. We denote the size of the evaluation set by $n = |\Xi|$ and label the points $\Xi = \{x_1, ..., x_n\}$. We note that operator $\rmT_\rmK$ is implicitly defined using the Lebesgue measure on $\calX$, but in the following it will be approximated by the empirical measure of our observation set $\Xi$, therefore the basis we obtain depends on this empirical distribution.

\begin{remark}
    We now present the following calculations for $\rset$-valued processes, noting that the $\rset^d$ case can be dealt with by stacking evaluations at the $n$ points of $\Xi$ into $\rset^{nd}$, thereby giving eigenvectors in $\rset^{nd}$ and forming a kernel matrix $\tilde{\rmK} \in \rset^{nd \times nd}$. Depending on the structure of the kernel over $d$ dimensional processes, the full $nd \times nd$ matrix may have particular structure which can ease computational costs. 
\end{remark}

We approximate the eigensystem condition using quadrature:
\begin{equation}
    \label{eqnapp_:eigensys}
\textstyle{    \lambda_m e_m(x) = \mathrm{T}_{\mathrm{K}} e_m(x) =\int_\calX \rmK(x, x')e_m(x') \rmd x' \approx \frac{1}{n}\sum_{i=1}^n \mathrm{K}(x, x_i)e_m(x_i).}
\end{equation}
Evaluating the above at the entire set of evaluation points
$\Xi$ results in the matrix eigenproblem
\[\textstyle{ \frac{1}{n} \tilde{\rmK} u_m = \hat{\lambda}_m u_m \eqsp,} \]
where \(\tilde{\rmK}\) is an $n \times n$ matrix with entries
\(\tilde{\rmK}_{i,j} = \mathrm{K}(x_i, x_j)\), \(\hat{\lambda}_m \in \rset\) is the
matrix eigenvalue and \(u_m \in \rset^n\) is the corresponding matrix
eigenvector, the set of which are orthonormal in $\rset^n$. By Baker's
theorem \citep[][Theorem 3.4]{baker1979Numerical}, \(\hat{\lambda}_m \to \lambda_m\) as
\(n \to \infty\).  

Now we use the Nystr\"{o}m method to extend the matrix eigenvectors to eigenfunctions $e_m \in \rmL^2(\calX, \rset^d)$. Firstly, we check the normality of the extended eigenvectors in $\rmL^2(\calX, \rset^d)$. If we take $e_m(x_i) = u_m^{(i)}$ i.e. the $i^\text{th}$ element of the vector $u_m$, then
$$
 \textstyle{\langle e_m, e_{m'}\rangle_{\rmL^2(\calX)} = \int_\calX e_m(x)e_{m'}(x) \rmd x \approx \frac{1}{n}\sum_{i=1}^n e_m(x_i)e_{m'}(x_i) = \frac{1}{n}u_m^Tu_{m'} = \frac{1}{n}\updelta_{mm'}}.
$$
Therefore, in order to satisfy orthonormality in $\rmL^2(\calX)$, we rescale the eigenvectors by $\sqrt{n}$, i.e. $e_m(x_i) = \sqrt{n}u_m^{(i)}$. Now we apply the Nystr\"{o}m formula to evaluate $e_m(x)$ for $x \in \calX$ outside the evaluation set $\Xi$. This gives
\begin{equation}
    \textstyle{ \hat{e}_m(x) \approx (n\hat{\lambda}_m)^{-1} \sum_{i=1}^{n} \rmK(x, x_i)e_m(x_i) = (\sqrt{n}\hat{\lambda}_m)^{-1} \sum_{i=1}^{n} \rmK(x, x_i)u_m^{(i)} \eqsp .}
\end{equation}

As $n \to \infty$, this method recovers the entire countable eigenbasis $(\lambda_m, e_m)_{m\in\mathbb{N}}$, however in practice we are limited to the first $n$ modes, of which the trailing modes may be poorly estimated. We therefore truncate the basis as follows. We define the truncation order $M$ as
$$\textstyle{ M = \min\{m :  \sum_{j=0}^m \hat{\lambda}_j /\sum_{j=0}^{n-1} \hat{\lambda}_j \geq \eta\}}
$$
where $\eta \in [0, 1]$ is a threshold parameter, fixed to $\eta = 0.99$ for our experimental results on the 1D datasets. We note that eigenvalues are arranged in descending order. Alternatively, one can treat the truncation order $M$ as a hyperparameter.

Having obtained an approximate eigensystem with truncation order $M$ as above, it is simple to obtain the spectral dataset $\calD^M = \{\bfZ^s\}_{s=1}^S = \left\{\{Z_m^s\}_{m=0}^{M}\right\}_{s=1}^S$ using
\begin{equation}Z_m^s = \lambda_m^{-1/2} \langle \bfY^s, e_m \rangle \approx \frac{1}{n^s} \sum_{i=1}^{n^s} \hat{\lambda}_m^{-1/2} \bfY^s(x_i^s) \hat{e}_m(x_i^s).
\end{equation}
Note that here we perform quadrature over the full evaluation set $\Xi^s$ for each sample $\bfY^s$ in the dataset, while the quadrature used in the Nystr\"{o}m formula was over the shared evaluation set $\Xi$.

\paragraph{Computational cost.} We now give a full picture of the computational costs in the full $\rset^d$ scenario, specifically for (a) solving the matrix eigenproblem, (b) projecting onto the spectral dataset, (c) reconstruction to the spatial domain. We note that (a) and (b) are one-off pre-processing costs while (c) should be counted as part of the sampling cost of our method.

Firstly (a), solving the matrix eigenproblem for a dense matrix $\tilde{\rmK} \in \rset^{nd \times nd}$ costs $\mathcal{O}((nd)^3)$. This is our main computational cost which, while only required as a pre-processing step, can be prohibitively expensive as $n$ and/or $d$ increase. We note that different kernel structures over $d$ dimensional processes can reduce this cost, for instance applying an independent kernel to each output dimension results in a block diagonal kernel matrix with eigenproblem cost of $\mathcal{O}(dn^3)$.  

Secondly (b), each coefficient $Z_m$ is an inner product over the $nd$-stacked vectors, computed for $M$ modes and $S$ dataset samples, giving cost $\mathcal{O}(S\;M\;nd)$.

Finally (c), reconstruction to the spatial domain depends on whether the query location $x_q$ is in the shared evaluation set $\Xi$. For $x_q \in \Xi$, $\bfY_{x_q}^M$ is a simple look-up costing $\mathcal{O}(Md)$. For $x_q \notin \Xi$, we apply the Nystr\"{o}m extension meaning each eigenfunction evaluation is an $nd$ inner product, giving the total cost $\mathcal{O}(Mnd^2)$.

\subsection{Eigenvalue-Weighted Score Matching Objective}
\label{app:eig_weighted_dsm}

We also propose the following weighted DSM loss to train the score network $\bm{s}_\theta: [0,T] \times \rset^{M+1} \rightarrow \R^{M+1}$:

\begin{equation} 
\label{eq:dsm_spectral}
\textstyle{\mathbb{E}_t \{ \mathbb{E}_{\bfZ_0, \bfZ_t} [ \| \tilde{\lambda}^\alpha \odot ( \bm{s}_\theta(\bfZ_t,t) - \nabla_{\bfZ_t} \log p(\bfZ_t|\bfZ_0) )\|^2_2  ] \} \eqsp ,}
\end{equation}
with $\tilde{\lambda}= [\lambda_0/\Lambda, \dots, \lambda_M/\Lambda]^\top$, \(\Lambda = \sum_{m=0}^{M} \lambda_m\), the vector of normalised eigenvalues, $\alpha>0$ a tuneable hyperparameter, $t \sim U([0, T])$, $\bfZ_0 = \{ \lambda_m^{-1/2} \int_\calX  \bfY_x^T e_m(x) \rmd x \}_{m=0}^{M}$ and
\begin{equation}
  \rmd \bfZ_t = -(1/2) \bfZ_t \rmd t +  \rmd \bfB_t \eqsp , 
\end{equation}
(we omit the parameter $\beta_t$). The exponentiation in \eqref{eq:dsm_spectral} is applied in a pointwise fashion and $\odot$ denotes element-wise multiplication. The motivation behind \eqref{eq:dsm_spectral} is to put more importance on fitting the lower frequency (i.e.\ higher eigenvalue) components of the decomposition $\sum_{m=0}^M {\lambda_m}^{1/2} Z_m e_m(x)$ since these matter the most in terms of quality of reconstruction, as per \Cref{thm:KL_theorem}. With $\alpha = 0$, $\lambda^\alpha = \bm{1}$, therefore we recover the standard DSM loss.

\subsection{Likelihood Evaluations}
\label{app:likelihoods}

Define the Mercer (equivalently Karhunen--Lo\`{e}ve with covariance kernel) \emph{recomposition} map $\phi: \R^{M+1} \to V \subset \mathrm{L}^2(\calX, \R^d)$, where $V = \mathrm{span}\{e_m\}_{m=0}^M$ is the $(M+1)$-dimensional subspace spanned by the leading eigenfunctions,
\begin{equation}
  \textstyle
  \phi(\bfZ) = \sum_{m=0}^{M} Z_m \sqrt{\lambda_m}\, e_m(\cdot),
  \qquad \bfZ = (Z_0, \dots, Z_M) \in \R^{M+1} .
\end{equation}
As $\phi$ is affine and the $e_m$ are linearly independent, $\phi$ is a bijection from $\R^{M+1}$ onto the subspace $V$. Its Jacobian is the linear map $J_z\phi : \R^{M+1} \to \mathrm{L}^2(\calX, \R^d)$ with columns $\tfrac{\partial \phi}{\partial Z_m} = \sqrt{\lambda_m}\, e_m$. Using the orthonormality of the eigenfunctions in $\mathrm{L}^2(\calX, \R^d)$, the associated Gram matrix is diagonal,
\begin{equation}
  \textstyle
  J_z\phi^\top J_z\phi
  = \big( \sqrt{\lambda_m}\sqrt{\lambda_n}\,
          \langle e_m, e_n \rangle_{\mathrm{L}^2(\calX,\R^d)} \big)_{m,n=0}^M
  = \big( \sqrt{\lambda_m}\sqrt{\lambda_n}\, \delta_{mn} \big)_{m,n=0}^M
  = \mathrm{diag}(\lambda_0, \dots, \lambda_M) .
\end{equation}

\begin{remark}
The image $F = \phi(\bfZ)$ is a \emph{function}, and the function space $\mathrm{L}^2(\calX, \R^d)$ admits no Lebesgue reference measure. Therefore there is no density of $F$ on the full function space. However, the image of $\phi$ is confined to the \emph{finite-dimensional} subspace $V$, on which the $(M+1)$-dimensional volume (Hausdorff) measure is a well-defined reference. All densities below are understood with respect to this measure on $V$.
\end{remark}

Because $\phi$ maps $\R^{M+1}$ into a higher-dimensional space, the relevant volume scaling is the rectangular change-of-variables factor $|J_z\phi^\top J_z\phi|^{1/2}$, rather than a square Jacobian determinant. The density of $F = \phi(\bfZ)$ on $V$ is therefore
\begin{equation}
  \textstyle
  \log p(F)
  = \log p(\bfZ) - \tfrac{1}{2}\log\big| J_z\phi^\top J_z\phi \big|
  = \log p(\bfZ) - \tfrac{1}{2} \sum_{m=0}^M \log \lambda_m ,
\end{equation}
where $\log p(\bfZ)$ is the joint log-density of the coefficients under the model. The coefficients $\bfZ$ are themselves modelled by a finite-dimensional diffusion model, which admits a likelihood via the probability flow ODE~\citep{song2020score}.

\subsection{Spectral Diffusion Model Algorithm}
\label{app:algorithm}

\begin{algorithm}[H]
\caption{\small Spectral Diffusion Model}
\label{alg:rsgm-c}
\begin{algorithmic}[1]
  \small
  \Require $T,\ \mathcal{D},\ \theta_0,\ N_{\mathrm{iter}},\ \vareps,\ \rmK,\ \eta$
  \State{{/// TRAINING ///}}
  \State Obtain $\mathcal{D}^M,\ \{(\hat\lambda_m, \hat e_m)\}_{m=0}^M$ from $\mathcal{D}$ via \Cref{alg:dataset_projection} \Comment{Spectral dataset projection}
  \For{$k \in \{0, \dots, N_{\mathrm{iter}}-1\}$}
    \State Sample mini-batch $\{\bfZ_0^s\}_{s \in B}$ from $\mathcal{D}^M$
    \State $t \sim \mathcal{U}(\ccint{\vareps, T})$
    \State $\bfZ_t^s = \rme^{-\frac12\int_0^t \beta_r \rmd r}\,\bfZ_0^s + \big(1 - \rme^{-\int_0^t \beta_r \rmd r}\big)^{1/2} G,\quad G \sim \mathrm{N}(0,\Id)$ \Comment{Forward diffusion}
    \State Compute denoising score-matching loss $\ell(\theta_k)$ \Comment{\Cref{eq:dsm_standard}}
    \State $\theta_{k+1} = \texttt{optimiser\_update}(\theta_k,\ \nabla_\theta \ell(\theta_k))$ \Comment{e.g.\ Adam}
  \EndFor
  \State $\theta^\star = \theta_{N_{\mathrm{iter}}}$
  \State{{/// SAMPLING ///}}
  \State $\overleftarrow{\bfZ}_0 \sim \mathrm{N}(0,\Id)$ \Comment{Initialise from the reference distribution}
  \State $b_{\theta^\star}(\bf{z}, t) = \beta_{T-t}\big(\tfrac12 \bf{z} + \bm{s}_{\theta^\star}(\bf{z},\, T-t)\big)$ for $t \in \ccint{0,T},\ \bf{z} \in \rset^{M+1}$ \Comment{Reverse-process drift}
  \State Simulate $\overleftarrow{\bfZ}_T$ from $\overleftarrow{\bfZ}_0$ by Euler--Maruyama ($N$ steps) with drift $b_{\theta^\star}$ and diffusion $\sqrt{\beta_{T-t}}$ \Comment{Approximate reverse diffusion}
  \State \textbf{return} $\theta^\star,\quad x \mapsto \sum_{m=0}^M \hat\lambda_m^{1/2}\,\overleftarrow{Z}_{m,T}\,\hat e_m(x)$ \Comment{Generative model in function space}
\end{algorithmic}
\end{algorithm}

\begin{algorithm}[H]
\caption{\small Spectral dataset projection}
\label{alg:dataset_projection}
\begin{algorithmic}[1]
  \small
  \Require $\mathcal{D} = \{\{\bfY^s(x_i^s)\}_{i=1}^{n^s}\}_{s=1}^S,\ \rmK,\ \eta$
  \State Form the shared evaluation set $\Xi = \bigcap_{s=1}^S \Xi^s = \{x_1, \dots, x_n\}$, with $n = |\Xi|$
  \State Gram matrix $\tilde{\rmK}_{ij} = \rmK(x_i, x_j)$ for $x_i, x_j \in \Xi$
  \State Solve eigenproblem $\tfrac{1}{n}\tilde{\rmK}\, u_m = \hat\lambda_m\, u_m$, with $\hat\lambda_0 \ge \hat\lambda_1 \ge \cdots$
  \State $\hat e_m(x) = (\sqrt{n}\,\hat\lambda_m)^{-1} \sum_{i=1}^n \rmK(x, x_i)\, u_m^{(i)}$ \Comment{Nystr\"om extension \citep[Theorem~3.4]{baker1979Numerical}}
  \State $M = \min\big\{m : \textstyle\sum_{j=0}^m \hat\lambda_j \,/\, \sum_{j=0}^{n-1} \hat\lambda_j \ge \eta\big\}$ \Comment{Truncation order}
  \State $Z_m^s = \tfrac{1}{n^s} \sum_{i=1}^{n^s} \hat\lambda_m^{-1/2}\, \bfY^s(x_i^s)\, \hat e_m(x_i^s)$ \quad for $s \in \{1,\dots,S\},\ m \in \{0,\dots,M\}$ \Comment{Project onto coefficients}
  \State \textbf{return} $\mathcal{D}^M = \{\bfZ^s\}_{s=1}^S = \{\{Z_m^s\}_{m=0}^M\}_{s=1}^S,\quad \{(\hat\lambda_m, \hat e_m)\}_{m=0}^M$
\end{algorithmic}
\end{algorithm}

\section{Preliminaries on Random Fields}
\label{sec:prel-rand-fields}

We will use two well-known results from the probability literature
which we recall for completeness, see \cite[Theorem
15.26]{charalambos2013infinite}.  Let
$\ensembleLigne{\mathcal{Y}_x}{x \in \mathcal{X}}$ be a family of spaces. For
any set $\mathsf{A} \subset \mathcal{X}$ and subset
$\mathsf{B} \subset \mathsf{A}$, we define
$\pi_{\mathsf{B}}^{\mathsf{A}}: \ \prod_{x \in \mathsf{A}} \mathcal{Y}^x \to
\prod_{x \in \mathsf{B}} \mathcal{Y}^x$ the projection operator and for any
$f \in \prod_{x \in \mathsf{A}} \mathcal{Y}^x$,
$\pi_{\mathsf{B}}^{\mathsf{A}}(f)$ is the restriction of $f$ to $\mathsf{B}$.

\begin{proposition}
  Let
  $\ensembleLigne{\mathcal{Y}_x, \mathcal{F}_x^{\mathcal{Y}}}{x \in
    \mathcal{X}}$ be a family of measurable spaces and for each finite subset
  $\mathsf{A} \subset \mathcal{X}$, let $\mu_{\mathsf{A}}$ be a probability
  measure on $\prod_{x \in \mathsf{A}} \mathcal{Y}_x$. Assume that
  $\ensembleLigne{\mu_{\mathsf{A}}}{\mathsf{A} \in 2^\mathcal{X}}$ is Kolmogorov
  consistent, i.e. for any $\mathsf{A} \subset \mathcal{X}$ and
  $\mathsf{B} \subset \mathsf{A}$ we have that
  $\mu_{\mathsf{B}} = (\pi_{\mathsf{B}}^{\mathsf{A}})_{\#}
  \mu_{\mathsf{A}}$. Then, there exists a unique probability measure $\mu$ on
  the product space $\prod_{x \in \mathcal{X}} \mathcal{Y}_x$ such that for any finite subset
  $\mathsf{A} \in \mathcal{X}$,
  $\mu_{\mathsf{A}} = (\pi_{\mathsf{A}}^{\mathcal{X}})_{\#} \mu$.
\end{proposition}

Note that in the case where $\mathcal{Y}_x = \rset^d$ and for any
$\mathsf{A} \subset \mathcal{X}$ finite, if we have that $\mu_{\mathsf{A}}$ is a
Gaussian probability measure with mean $m_{\mathsf{A}}$ and covariance matrix
$\Sigma_{\mathsf{A}}$, then the previous proposition can be simplified.

\begin{proposition}
  For each finite subset $\mathsf{A} \subset \mathcal{X}$, let
  $\mu_{\mathsf{A}}$ be a Gaussian probability measure on
  $\prod_{x \in \mathsf{A}} \rset^d$ with mean $m_{\mathsf{A}}$ and covariance
  matrix $\Sigma_{\mathsf{A}}$. Assume that there exist
  $m: \ \mathcal{X} \to \rset^d$ and
  $\Sigma: \ \mathcal{X} \times \mathcal{X} \to \rset^{d \times d}$ such that
  for any finite subset $\mathsf{A} \subset \mathcal{X}$ we have
  $m_{\mathsf{A}}$ is the restriction of $m$ to $\mathsf{A}$ and
  $\Sigma_{\mathsf{A}}$ is the restriction of $\Sigma$ to $\mathsf{A}$.  Then,
  there exists a unique probability measure $\mu$ on the product space
  $\prod_{x \in \mathcal{X}} \rset^d$ such that for any finite subset
  $\mathsf{A} \in \mathcal{X}$,
  $\mu_{\mathsf{A}} = (\pi_{\mathsf{A}}^{\mathcal{X}})_{\#} \mu$.
\end{proposition}

\section{Gaussian Process and Reference Measure}
\label{sec:proof-prop-limiting}

In this section, we provide two proofs of \Cref{prop:proj_gaussian_process}.

\subsection{Partition Definition}
\label{sec:partition-definition}

We recall \Cref{prop:proj_gaussian_process}.
\begin{proposition}{}{}
  \label{prop:proj_gaussian_process_appendix1}
  Let $\pi^0$ be the distribution of $\sum_{m=0}^{+\infty} \lambda_m^{1/2} Z_{m, 0} e_m$
  and $\pi$ the target distribution. Denote $\mathrm{GP}(\calX)$ the space of
  Gaussian processes on $\calX$ and assume that $\rmK$ is the covariance
  kernel. Then,
  $\pi^0 \in \argmin_{\pi_{\mathrm{GP}} \in \mathrm{GP}(\calX)}
  \KL{\pi}{\pi_{\mathrm{GP}}}$. In addition, $\sum_{m=0}^{M} \lambda_m^{1/2} Z_{m, 0} e_m$
  is the projection of $\sum_{m=0}^{+\infty} \lambda_m^{1/2} Z_{m, 0} e_m$ on the subspace
  of $\rmL^2(\calX)$ spanned by $\{e_m\}_{m=0}^M$.
\end{proposition}

Note that in this first approach we do not assume any regularity on the samples
of the process. This proof is based on \citet[Theorem
1]{sun2019functional} which considered a modified definition of the
Kullback-Leibler divergence. Namely, for any $\mu, \nu$ probability measures
over a probability space $(\Omega, \calF)$ we define
\begin{equation}
  \textstyle{\KL{\mu}{\nu} = \sup \ensembleLigne{\KL{\mu_P}{\nu_P}}{P \ \text{ measurable finite partition of $\Omega$}}} \eqsp ,
\end{equation}
where for any measurable finite partition $P = \{\Omega_1, \dots, \Omega_N\}$ we
define $\mu_P = \{\mu(\Omega_1), \dots, \mu(\Omega_N)\}$, similarly for $\nu_P$.
With that definition \citet[Theorem 1]{sun2019functional} holds. Namely, for
any $\mu, \nu \in \Pens((\rset^d)^\calX)$\footnote{The product space
  $(\rset^d)^\calX$ is a measurable space with the cylindrical sigma algebra,
  see \cite{sun2019functional} for more details in that context.}, we have
\begin{equation}
  \label{eq:supremum_kl_process}
  \textstyle{ \KL{\mu}{\nu} = \sup\ensembleLigne{\KL{\mu_{x_1, \dots, x_n}}{\nu_{x_1, \dots, x_n}}}{(x_1, \dots, x_n) \in \calX^n, \ n \in \nset} \eqsp . }
\end{equation}
Let $\pi$ be the target distribution of the stochastic process. We define
$\mathrm{GP}(\calX)$ the space of Gaussian processes over the input space
$\calX$. We have that
\begin{align}
  &\inf\ensembleLigne{\sup\ensembleLigne{\KL{\pi_{x_1, \dots, x_n}}{\pi^0_{x_1, \dots, x_n}}}{(x_1, \dots, x_n) \in \calX^n, \ n \in \nset}}{\pi^0 \in \mathrm{GP}(\calX)} \\
  & \qquad \geq \sup\ensembleLigne{\inf\ensembleLigne{\KL{\pi_{x_1, \dots, x_n}}{\pi^0_{x_1, \dots, x_n}}}{\pi^0 \in \mathrm{GP}(\calX)}}{(x_1, \dots, x_n) \in \calX^n, \ n \in \nset} \eqsp . 
\end{align}
In addition, we have
\begin{align}
  &\sup\ensembleLigne{\inf\ensembleLigne{\KL{\pi_{x_1, \dots, x_n}}{\pi^0_{x_1, \dots, x_n}}}{\pi^0 \in \mathrm{GP}(\calX)}}{(x_1, \dots, x_n) \in \calX^n, \ n \in \nset} \\
  &\qquad  = \sup\ensembleLigne{\inf\ensembleLigne{\KL{\pi_{x_1, \dots, x_n}}{\pi^0_{x_1, \dots, x_n}}}{\pi^0_{x_1, \dots, x_n} \in \mathrm{GP}(\{x_1, \dots, x_n\})}}{ \\ & \qquad \qquad (x_1, \dots, x_n) \in \calX^n, \ n \in \nset} \eqsp \notag ,
\end{align}
where we emphasise that the set $\mathrm{GP}(\{x_1, \dots, x_n\})$ is simply the
set of $nd$-dimensional Gaussian probability measures. Hence, for any
$\{x_1, \dots, x_n\} \in \calX^n$ and $n \in \nset$, we have that
\begin{equation}
  \KL{\pi_{x_1, \dots, x_n}}{\pi^{0,\star}_{x_1, \dots, x_n}} = \inf\ensembleLigne{\KL{\pi_{x_1, \dots, x_n}}{\pi^0_{x_1, \dots, x_n}}}{\pi^0_{x_1, \dots, x_n} \in \mathrm{GP}(\{x_1, \dots, x_n\})} \eqsp , 
\end{equation}
where $\pi^{0,\star}_{x_1, \dots, x_n}$ is the Gaussian measure with same mean
and covariance matrix as $\pi_{x_1, \dots, x_n}$. Using the Kolmogorov extension
theorem \cite[Theorem 15.26]{charalambos2013infinite}, there exists
$\pi^\star \in \Pens((\rset^d)^\calX)$ such that for any
$\{x_1, \dots, x_n\} \in \calX^n$ and $n \in \nset$,
$\pi^{\star}_{x_1, \dots, x_n} = \pi^{0,\star}_{x_1, \dots, x_n}$. Therefore, we have that
\begin{equation}
  \KL{\pi_{x_1, \dots, x_n}}{\pi^{\star}_{x_1, \dots, x_n}} = \inf\ensembleLigne{\KL{\pi_{x_1, \dots, x_n}}{\pi^0_{x_1, \dots, x_n}}}{\pi_{x_1, \dots, x_n} \in \mathrm{GP}(\calX)} \eqsp .
\end{equation}
Hence, using \cref{eq:supremum_kl_process} we get that
\begin{align}
  & \sup\ensembleLigne{\inf\ensembleLigne{\KL{\pi_{x_1, \dots, x_n}}{\pi^0_{x_1, \dots, x_n}}}{\pi^0 \in \mathrm{GP}(\calX)}}{(x_1, \dots, x_n) \in \calX^n, \ n \in \nset} \\
  & \qquad = \sup\ensembleLigne{\KL{\pi_{x_1, \dots, x_n}}{\pi^{\star}_{x_1, \dots, x_n}}}{\{x_1, \dots, x_n\} \in \calX^n}  = \KL{\pi}{\pi^\star} \eqsp . 
\end{align}
Therefore, we have that
\begin{equation}
  \inf\ensembleLigne{\sup\ensembleLigne{\KL{\pi_{x_1, \dots, x_n}}{\pi^0_{x_1, \dots, x_n}}}{(x_1, \dots, x_n) \in \calX^n, \ n \in \nset}}{\pi^0 \in \mathrm{GP}(\calX)} \geq \KL{\pi}{\pi^\star} \eqsp ,
\end{equation}
which implies using \cref{eq:supremum_kl_process} that
\begin{equation}  
  \KL{\pi}{\pi^\star} \leq \inf\ensembleLigne{\KL{\pi}{\pi^0}}{\pi^0 \in \mathrm{GP}(\calX)} \eqsp . 
\end{equation}
The equality holds since $\pi^\star \in \mathrm{GP}(\calX)$. Finally, since
$\pi^\star$ and $\pi$ have the same covariance kernels, they have the same
Karhunen-Lo\`{e}ve eigensystems. Therefore, there exists $\{Z_m\}_{m \in \nset}$
i.i.d.\ Gaussian random variables with zero mean and identity covariance matrix
such that $\bfY^\star = \sum_{m \in \nset} \lambda_m^{1/2} Z_m e_m$ has
distribution $\pi^\star$, which concludes the proof.

\subsection{Sample Continuous}
\label{sec:sample-continuous}

In our second approach, we restrict ourselves to the case of sample continuous
processes. Namely, we no longer consider $\mu \in \Pens((\rset^d)^\calX)$ but
$\mu \in \Pens(\rmc(\calX, \rset^d))$.
\begin{proposition}{}{}
  \label{prop:proj_gaussian_process_appendix2}
  Assume that $\pi \in \rmc(\calX, \rset^d)$ and that there exists
  $\phi: \ \coint{0, +\infty}$ such that for any $x_1, x_2 \in \calX$,
\begin{equation}
  \textstyle{\expeLigne{\normLigne{\bfY_{x_1} - \bfY_{x_2}}^2} \leq \phi(\normLigne{x_1 - x_2}) \eqsp ,}
\end{equation}
such that $\int_0^{+\infty} \phi(\exp[-t^2]) \rmd t < +\infty$.  Let $\pi^0$ be
the distribution of $\sum_{m=0}^{+\infty} \lambda_m^{1/2} Z_{m, 0}
e_m$. Denote $\mathrm{GP}(\calX)$ the space of Gaussian processes on $\calX$ and
assume that $\rmK$ is the covariance kernel. Then,
$\pi^0 \in \argmin_{\pi_{\mathrm{GP}} \in \mathrm{GP}(\calX)}
\KL{\pi}{\pi_{\mathrm{GP}}}$. In addition,
$\sum_{m=0}^{M} \lambda_m^{1/2} Z_{m, 0} e_m$ is the projection of
$\sum_{m=0}^{+\infty} \lambda_m^{1/2} Z_{m, 0} e_m$ on the subspace of
$\rmL^2(\calX)$ spanned by $\{e_m\}_{m=0}^M$.

\end{proposition}

In that case, we show that
\begin{equation}
  \textstyle{ \KL{\mu}{\nu} = \sup\ensembleLigne{\KL{\mu_{x_1, \dots, x_n}}{\nu_{x_1, \dots, x_n}}}{(x_1, \dots, x_n) \in \calX^n, \ n \in \nset} \eqsp . }
\end{equation}

More precisely, we have the following proposition.

\begin{proposition}{}{}
  Let  $\mu, \nu \in \Pens(\rmc(\calX, \rset^d))$
\begin{equation}
  \label{eq:supremum_kl_process_duo}
  \textstyle{ \KL{\mu}{\nu} = \sup\ensembleLigne{\KL{\mu_{x_1, \dots, x_n}}{\nu_{x_1, \dots, x_n}}}{(x_1, \dots, x_n) \in \calX^n, \ n \in \nset} \eqsp . }
\end{equation}  
\end{proposition}

\begin{proof}
  First, note that for any $\{x_1, \dots, x_n\} \in \calX^n$, $n \in \nset$ we
  have using the data processing theorem \cite[Lemma 9.4.5]{ambrosio200gradient}
  \begin{equation}
    \label{eq:ineq_sens_uno}
    \KL{\mu}{\nu} \geq \KL{\mu_{x_1, \dots, x_n}}{\nu_{x_1, \dots, x_n}} \eqsp . 
  \end{equation}
  Since $\calX$ is compact, for any $n \in \nset$, there exists
  $\{x_1, \dots, x_n\}$ such that
  $\calX \subset \cup_{k=1}^n \ball{x_k}{1/(n+1)}$.  For any $n \in \nset$, let
  $\{\varphi_k\}_{k=1}^n$ be a smooth partition of unity associated with
  $\{\ball{x_k}{1/(n+1)}\}_{k=1}^n$.  For any $\rmc(\calX, \rset^d)$-valued
  random variable $\bfY$ and $n \in \nset$, denote $\bfY^n$ such that for any $x \in \calX$
  \begin{equation}
    \textstyle{\bfY^n_x = \sum_{k=1}^n \varphi_k(x) \bfY_{x_k} \eqsp . }
  \end{equation}
  Since $\bfY$ is continuous and $\mcx$ is compact we have that $\bfY$ is
  uniformly continuous. Hence, for any $\vareps > 0$, there exists $n \in \nset$
  such that for any $x_1, x_2$, $\normLigne{x_1 - x_2} \leq 1/(n+1)$,
  $\normLigne{\bfY_{x_1} - \bfY_{x_2}} \leq \vareps$. Therefore, we have that
  \begin{equation}
    \textstyle{\normLigne{\bfY^n_x - \bfY_x} \leq  \sum_{k=1}^n \varphi_k(x) \normLigne{\bfY_{x_k} - \bfY_x} \leq \vareps \eqsp . }
  \end{equation}
  Therefore, we have that
  $\lim_{n \to + \infty} \sup_{x \in \calX} \normLigne{\bfY^n_x - \bfY_x} =
  0$. Therefore, for any $n \in \nset$, denoting $\mu_n$ the distribution of
  $\bfY^n$, we get that $(\mu_n)_{n \in \nset}$ converges to $\mu$ in
  $\Pens(\rmc(\calX, \rset^d))$. For any $n \in \nset$, let
  $\bfa = \{a_k\}_{k=1}^n \in (\rset^d)^n$ and $f_n^\bfa: \ \calX \to \rset^d$
  such that for any $x \in \calX$,
  $f_n^\bfa(x) = \sum_{k=1}^n a_k \varphi_k(x)$. Denote
  $\rmc_n(\calX, \rset^d) = \ensembleLigne{f_n^\bfa}{\bfa \in
    (\rset^d)^n}$. Define $\varphi_n: \ (\rset^d)^n \to \rmc_n(\calX, \rset^d)$
  such that for any $\bfa \in (\rset^d)^n$, $\varphi_n(\bfa) = f_n^\bfa$. We
  have that $\varphi_n$ is a bijection. In addition, we have that
  $(\varphi_n)_\# \mu_n = \mu_{x_1, \dots, x_n}$. Therefore, using the data
  processing inequality we have that for any $n \in \nset$
  \begin{equation}
    \KL{\mu_n}{\nu_n} = \KL{\mu_{x_1, \dots, x_n}}{\nu_{x_1, \dots, x_n}} \eqsp . 
  \end{equation}
  In addition, since $(\mu_n, \nu_n) \to (\mu, \nu)$ we have using that the
  Kullback-Leibler divergence is lower semi-continuous \cite[Lemma
  1.4.3]{dupuis2011weak}
  \begin{equation}
    \textstyle{ \lim_{n \to +\infty} \KL{\mu_n}{\nu_n} = \lim_{n \to +\infty} \KL{\mu_{x_1, \dots, x_n}}{\nu_{x_1, \dots, x_n}} \geq \KL{\mu}{\nu} \eqsp . }
  \end{equation}
  Therefore, we get that
  \begin{equation}
      \textstyle{ \KL{\mu}{\nu} \leq \sup\ensembleLigne{\KL{\mu_{x_1, \dots, x_n}}{\nu_{x_1, \dots, x_n}}}{(x_1, \dots, x_n) \in \calX^n, \ n \in \nset} \eqsp . }
  \end{equation}
  Combining this result and \cref{eq:ineq_sens_uno}, we conclude the proof.
\end{proof}

The rest of the proof is similar to \Cref{sec:partition-definition}, except that
one needs to check that the obtained Gaussian process $\pi^\star$ is sample
continuous (up to a modification). This is done using that there exists
$\phi: \ \coint{0, +\infty}$ such that for any $x_1, x_2 \in \calX$,
\begin{equation}
  \textstyle{\expeLigne{\normLigne{\bfY_{x_1} - \bfY_{x_2}}^2} \leq \phi(\normLigne{x_1 - x_2}) \eqsp ,}
\end{equation}
such that $\int_0^{+\infty} \phi(\exp[-t^2]) \rmd t < +\infty$, see
\cite{fernique1975regularite} and that the obtained Gaussian process with
distribution $\pi^\star$ has same mean and covariance as $\pi^0$.

\section{Induced Spatial Process}
\label{sec:spatial_process}
In this section we show that our spectral decomposition implicitly induces a
process on the original spatial space with a diffusion coefficient given by our
choice of kernel.

The strategy of the proof is as follows. First, we show the existence of a
doubly stochastic process (time/input) such that its ``horizontal'' and
``vertical'' slices are correct. More precisely, we define a doubly stochastic
process such that for a fixed \emph{finite} number of positions, it corresponds
to the solution of some stochastic differential equation. In addition, for any
fixed \emph{finite} time stamps, it is a Gaussian process.

In order to correctly handle solutions of stochastic differential equations we
\emph{do not} consider a process on
$\rset^{\coint{0,+\infty} \times \mathcal{X}}$, which is not regular enough, but
on $\rmc(\coint{0,+\infty}, \rset)^{\mathcal{X}}$. This process is defined in
\Cref{sec:doubly-index-proc}. 

Once we are equipped with that doubly indexed stochastic process, we can define
another doubly stochastic process in \Cref{sec:multivariate_ou}, denoted $\bfY$
such that for any finite number of input points it corresponds to a (correlated)
Ornstein-Uhlenbeck. In addition, for any \emph{finite} time stamps it is a
spatial stochastic process interpolating between the input stochastic process
and a Gaussian process.

In \Cref{sec:induced-process}, we conclude by showing that the distribution of
the reconstructed process coincides with $\bfY$ if we let the number of spectral
components $M \to +\infty$. 

\subsection{A Doubly Indexed Process}
\label{sec:doubly-index-proc}

First, we define a stochastic process doubly indexed by the time ($t \geq 0$)
and by the input component ($x \in \mathcal{X}$). Consider the family of
distribution
$\ensembleLigne{\eta_{\mathsf{A}}}{\mathsf{A}=\{ x_i\}_{i=1}^n, x_i \in \mathcal{X}, \ n
  \in \nset}$ such that for any $\mathsf{A}=\{ x_i\}_{i=1}^n$ we have that
$\eta_{\mathsf{A}}$ is the distribution of 
$\{\{\rmK^{1/2}(\mathsf{A}, \mathsf{A})(\int_0^t g(s) \rmd \bfB_s^i)_{i=1}^n\}_{t \geq 0}\}_{i=1}^n$
with $((\bfB_t^i)_{i=1}^n)_{t \geq 0}$ a family of $n$ independent Brownian
motions and $g: \ \coint{0, +\infty} \to \coint{0, +\infty}$ given for any $t \geq 0$ by
\begin{equation}
\textstyle{g(t) = \beta_t^{1/2} \exp[\int_0^t \beta_s \rmd s / 2] . }
\end{equation}
We also denote $G(t) = \int_0^t g(s)^2 \rmd s = \rme^{\int_0^t \beta_s \rmd s} - 1$ for any $t \geq 0$. 
Using the Kolmogorov extension theorem, we have that $\eta_\mathsf{A}$
is uniquely defined by a mean function
$m_{\mathsf{A}}: \ \coint{0,+\infty} \to \rset$ and a covariance matrix
function $ \Sigma_{\mathsf{A}}: \ \coint{0,+\infty} \to \rset$ such that for
any $t \geq 0$,
\begin{equation}
  m_{\mathsf{A}}(t) = 0 , \qquad \Sigma_{\mathsf{A}}(s, t) = \min(G(s),G(t)) \rmK(\mathsf{A}, \mathsf{A}) . 
\end{equation}
Let $\mathsf{B} \subset \mathsf{A}$ and we have that $\eta_\mathsf{B}$
is uniquely defined by a mean function
$m_{\mathsf{B}}: \ \coint{0,+\infty} \to \rset$ and a covariance matrix
function $ \Sigma_{\mathsf{B}}: \ \coint{0,+\infty} \to \rset$ such that for
any $t \geq 0$,
\begin{equation}
  m_{\mathsf{B}}(t) = 0 , \qquad \Sigma_{\mathsf{B}}(s, t) = \min(G(s),G(t)) \rmK(\mathsf{B}, \mathsf{B}) . 
\end{equation}
We have that $(\pi_{\mathsf{B}}^{\mathsf{A}})_{\#} \eta_{\mathsf{A}}$ has mean
function $m_{\mathsf{B}}$ and covariance matrix function
$\Sigma_{\mathsf{B}}$. Therefore, we have that
$(\pi_{\mathsf{B}}^{\mathsf{A}})_{\#} \eta_{\mathsf{A}} =
\eta_{\mathsf{B}}$. Using the Kolmogorov extension theorem there exists
$\eta \in \mathcal{P}(\rmc(\coint{0,+\infty}, \rset)^\mathcal{X})$ such that for
any for any $\mathsf{A}=\{ x_i\}_{i=1}^n$,
$(\pi_{\mathsf{A}})_{\#} \eta = \eta_{\mathsf{A}}$. We denote $\mathbf{H}$ a
random variable with distribution $\eta$. We have that
\begin{enumerate}[wide, labelwidth=!, labelindent=0pt, label=(\alph*)]
\item Since $\rmK(x,x) = 1$, for any $x \in \mathcal{X}$,
  $\{\mathbf{H}_{t}(x)\}_{t \geq 0}$ has distribution $\{\int_0^t g(s) \rmd \bfB_s\}_{t \geq 0}$.
\item For any
$\mathsf{A} = \{x_1, \dots, x_n\}$,
$\{\{\mathbf{H}_{t}(x_i)\}_{i=1}^n\}_{t \geq 0}$ has the same distribution as
$\{\{\rmK^{1/2}(\mathsf{A}, \mathsf{A})\int_0^t g(s) \rmd (\bfB_s^i)_{i=1}^n\}_{i=1}^n\}_{t \geq 0}$
with $((\bfB_t^i)_{i=1}^n)_{t \geq 0}$ a family of $n$-independent Brownian
motions.
\item For any $t \geq 0$ we have that
$\{\mathbf{H}_{t}(x)\}_{x \in \mathcal{X}}$ is a Gaussian process with zero mean and kernel
$G(t) \rmK$.
\end{enumerate}

\subsection{Multivariate Ornstein-Uhlenbeck Process}
\label{sec:multivariate_ou}

\begin{proposition}{}{}
  Assume that $(\bfY_0(x))_{x \in \mathcal{X}}$ is a stochastic process and
  $\mu: \ \mathcal{X} \to \rset$.  Let $\bfY$ be a stochastic process on
  $\coint{0,+\infty}\times \mathcal{X}$ given for any
  $t, x \in \coint{0,+\infty}\times \mathcal{X}$ by
  \begin{equation}
    \label{eq:definition_Y}
    \textstyle{
      \bfY_t(x) =  \rme^{-\int_{0}^t \beta_s ds/2} ~\bfY_0(x) + (1 - \rme^{-\int_{0}^t \beta_s ds/2} ) \mu(x) + \rme^{-\int_0^t \beta_s \rmd s/2} \mathbf{H}_t(x) .
      }
    \end{equation}
    We have that
  for any $x \in \mathcal{X}$,
  $\bfY_t(x) \xrightarrow[t \rightarrow 0]{a.s.} \bfY_0(x)$ and
  $\bfY_t(x) \xrightarrow[t \rightarrow \infty]{a.s.} \bfY_\infty(x)$.
  In addition, we have that
  \begin{enumerate}[label=(\alph*)]
  \item For any $t \geq 0$, conditionally to $\mathbf{Y}_0$, $\mathbf{Y}_t$ is a
    Gaussian process with mean
    $\mu_t = e^{- \int_{0}^t \beta_s ds / 2} ~\mathbf{Y}_0 + (1 -
      \rme^{- \int_{0}^t \beta_s ds/2} ) \mu $ and kernel $\Sigma_t = (1 - \rme^{-\int_{0}^t \beta_s ds} ) \rmK$.
    \item For any $\x = \{x_1, \dots, x_n\} \in \mathcal{X}$, $\{\mathbf{Y}_t(\x)\}_{t \geq 0}$ is a (weak) solution to
      \begin{equation}\label{eq:forward_SDE_sp}
  \rmd \tilde{\bfY}_t(\x) = (\mu(\x) -\tilde{\bfY}_t(\x))~(\beta_t/2) \rmd t + \sqrt{\beta_t {\rmK}(\x,\x)}  \rmd \bfB_t  \eqsp .
\end{equation}
  \end{enumerate}
\end{proposition}

\begin{proof}
  We divide the proof into two parts.
  \begin{enumerate}[wide, labelwidth=!, labelindent=0pt, label=(\alph*)]
  \item       First, we show that for any $t \geq 0$, conditionally to $\mathbf{Y_0}$,
    $\mathbf{Y}_t$ is a Gaussian process with mean
    $\mu_t = e^{-\int_{0}^t \beta_s ds/2} ~\mathbf{Y}_0 + (1 -
    \rme^{-\int_{0}^t \beta_s ds / 2} ) \mu $ and kernel
    $\Sigma_t = (1 - \rme^{-\int_{0}^t \beta_s ds} ) k$. This is a direct
    consequence of the definition of $\mathbf{H}$.
  \item Second, we show that for any $\x = \{x_1, \dots, x_n\} \in \mathcal{X}$, $\{\mathbf{Y}_t(\x)\}_{t \geq 0}$ is a solution to
      \begin{equation*}
  \rmd \tilde{\bfY}_t(\x) = (\mu(\x) -\tilde{\bfY}_t(\x))(\beta_t/2) \rmd t + \sqrt{\beta_t {\rmK}(\x,\x)}  \rmd \bfB_t  \eqsp .
\end{equation*}
First, we have that
  \begin{equation}
    \textstyle{
      \bfY_t(\x) =  \rme^{- \int_{0}^t \beta_s ds/2} ~\bfY_0(\x) + (1 - \rme^{- \int_{0}^t \beta_s ds/2} ) \mu(\x) + \rme^{-\int_0^t \beta_s \rmd s / 2}\mathbf{H}_t(\x) .
      }
    \end{equation}
    By definition of $\mathbf{H}$, this can be rewritten as
  \begin{align}
      \bfY_t(\x) &=      \textstyle{\rme^{- \int_{0}^t \beta_s ds/2} ~\bfY_0(\x) + (1 - \rme^{- \int_{0}^t \beta_s ds/2} ) \mu(\x) + \int_0^t \beta_s^{1/2} \rme^{-\int_s^t \beta_u \rmd u /2} \rmK(\x,\x)^{1/2} \rmd \bfB_s .
                   } \\
    &=      \textstyle{\rme^{- \int_{0}^t \beta_s ds / 2} ~\bfY_0(\x) + (1 - \rme^{- \int_{0}^t \beta_s ds / 2} ) \mu(\x) + \rme^{-\int_0^t \beta_u \rmd u /2} \int_0^t \beta_s^{1/2} \rme^{\int_0^s \beta_u \rmd u /2} \rmK(\x,\x)^{1/2} \rmd \bfB_s .
      }
  \end{align}
  For any $t \geq 0$, we have by (stochastic) integration by parts 
  \begin{align*}
    \textstyle{\int_0^t (\beta_s/2) (\mu(\x) - \bfY_s(\x))} \rmd s &= \textstyle{\int_0^t ((-\beta_s/2)\rme^{-\int_0^s \beta_u \rmd u /2})(\bfY_0(\x)-\mu(\x)) \rmd s}
    \\
                                                                   & \qquad + \textstyle{\int_0^t (-\beta_s/2)\rme^{-\int_0^s \beta_u \rmd u/2} \int_0^s \beta_u^{1/2} \rme^{\int_0^u \beta_v \rmd v /2} \rmK(\x,\x)^{1/2} \rmd \bfB_u \rmd s}  \\
    &= \textstyle{\int_0^t ((-\beta_s/2)\rme^{-\int_0^s \beta_u \rmd u /2})(\bfY_0(\x)-\mu(\x)) \rmd s}
    \\
                                                                   & \qquad - \textstyle{\int_0^t \rme^{-\int_0^s \beta_u \rmd u/2} \beta_s^{1/2} \rme^{\int_0^s \beta_u \rmd u /2} \rmK(\x,\x)^{1/2} \rmd \bfB_s} \\
                                                                   & \qquad \textstyle{+ \rme^{-\int_0^t \beta_s\rmd s / 2} \int_0^t \beta_s^{1/2} \rme^{\int_0^s \beta_u \rmd u / 2} \rmK(\x, \x)^{1/2} \rmd \bfB_s. } \\
    &= \textstyle{\int_0^t ((-\beta_s/2)\rme^{-\int_0^s \beta_u \rmd u /2})(\bfY_0(\x)-\mu(\x)) \rmd s}
    \\
                                                                   & \qquad - \textstyle{\int_0^t  \beta_s^{1/2}  \rmK(\x,\x)^{1/2} \rmd \bfB_s + \rme^{-\int_0^t \beta_s\rmd s / 2} \int_0^t \beta_s^{1/2} \rme^{\int_0^s \beta_u \rmd u / 2} \rmK(\x, \x)^{1/2} \rmd \bfB_s }
\\
    &= \textstyle{(\rme^{-\int_0^t \beta_u \rmd u /2} - 1)(\bfY_0(\x)-\mu(\x))}
    \\
                                                                   & \qquad - \textstyle{\int_0^t  \beta_s^{1/2}  \rmK(\x,\x)^{1/2} \rmd \bfB_s +  \int_0^t \beta_s^{1/2} \rme^{-\int_s^t \beta_u \rmd u / 2} \rmK(\x, \x)^{1/2} \rmd \bfB_s. }                                                                     
  \end{align*}
  Therefore, we have that
  \begin{align}
    &\textstyle{\bfY_0(\x) + \int_0^t (\beta_s/2)(\mu(\x) - \bfY_s(\x)) \rmd s + \int_0^t \beta_s^{1/2} \rmK(\x,\x)^{1/2} \rmd \bfB_s} \\
    & \qquad \qquad = \textstyle{\rme^{- \int_{0}^t \beta_s ds / 2} ~\bfY_0(\x) + (1 - \rme^{- \int_{0}^t \beta_s ds / 2} ) \mu(\x) + \rme^{-\int_0^t \beta_u \rmd u /2} \int_0^t \beta_s^{1/2} \rme^{\int_0^s \beta_u \rmd u /2} \rmK(\x,\x)^{1/2} \rmd \bfB_s .
      }
  \end{align}
Therefore
\begin{equation}
  \bfY_t(\x) = \textstyle{\bfY_0(\x) + \int_0^t (\beta_s/2)(\mu(\x) - \bfY_s(\x)) \rmd s + \int_0^t \beta_s^{1/2} \rmK(\x,\x)^{1/2} \rmd \bfB_s} . 
\end{equation}
Hence we have that $(\bfY_t(\x))_{t \geq 0}$ is a solution of \eqref{eq:forward_SDE_sp}, see \citet[Section 3, Definition 3.1]{karatzas1991brownian}.
  \end{enumerate}
\end{proof}

\subsection{Induced Process}
\label{sec:induced-process}

Now we aim at showing that the distribution in the spatial domain induced by the Ornstein-Uhlenbeck process in the spectral space (as per our introduced method) is the same process as in \cref{sec:multivariate_ou}.

For any $M \in \nset$, we define the collection of random variables
$\{\mathbf{H}_t^M(x)\}_{t \geq 0, x \in \mathcal{X}}$ given by
$\mathbf{H}_t^M(x) = \sum_{m=0}^M \lambda_m^{1/2} \int_0^t g(s) \rmd \bfB_s^m e_m(x)$.  Note that
$\mathbf{H}^M$ is a $\rmc(\coint{0,+\infty}, \rset)^{\mathcal{X}}$-valued random
variable.

\begin{proposition}{}{}
  For any $M \in \nset$, denote
  $\eta^M \in \mathcal{P}(\rmc(\coint{0,+\infty}, \rset)^{\mathcal{X}})$ the
  distribution of $\mathbf{H}^M$. Then, we have that
  $\lim_{M \to +\infty} \eta^M = \eta$, where $\eta$ is the distribution of $\mathbf{H}$.
\end{proposition}

\begin{proof}
  In order to prove that $\lim_{M \to +\infty} \eta^M = \eta$, we only need to
  prove that for any $\x = \{x_1, \dots, x_n\}$ with $n \in \nset$ and
  $x_1, \dots, x_n \in \mathcal{X}$ we have that
  $(\mathbf{H}^M(x_1), \dots, \mathbf{H}^M(x_n)) \to (\mathbf{H}(x_1), \dots,
  \mathbf{H}(x_n))$ in distribution (in the space
  $\mathcal{P}(\rmc(\coint{0,+\infty}, \rset)^n)$). Let
  $\x = \{x_1, \dots, x_n\}$ with $n \in \nset$ and
  $x_1, \dots, x_n \in \mathcal{X}$. We show that
  $(\mathbf{H}^M(x_1), \dots, \mathbf{H}^M(x_n))$ is a Cauchy sequence in
  probability. We recall that $\rmc(\coint{0,+\infty}, \rset)^n$ is a metric
  space with distance $d$ given for any $f = (f^1, \dots, f^n)$,
  $g = (g^1, \dots, g^n) \in \rmc(\coint{0,+\infty}, \rset)^n$ by 
  \begin{equation}
    \textstyle{d(f,g) = \sum_{\ell=1}^n \sum_{j \in \nset} 2^{-j} \| f_{\ccint{0,j}}^\ell - g_{\ccint{0,j}}^\ell\|_\infty / (1 + \| f_{\ccint{0,j}}^\ell - g_{\ccint{0,j}}^\ell\|_\infty) , }
  \end{equation}
  where $f_{\ccint{0,a}}$ is the restriction of $f$ on $\ccint{0,a}$. In what
  follows, we assume without loss of generality that $n = 1$ and $\x =
  \{x\}$. Let $\vareps >0$ and $M, N \in \nset$. We have
  \begin{equation}
    \textstyle{
      d(\mathbf{H}^M(x), \mathbf{H}^{M+N}(x)) \leq \sum_{j \in \nset} 2^{-j}\sum_{\ell=M+1}^{M+N} \lambda_\ell^{1/2} e_\ell(x) \sup_{t \in \ccint{0,j}} |\int_0^t g(s) \rmd \bfB_s^\ell| .
      }
    \end{equation}
    Using the fact that $( |\int_0^t g(s) \rmd \bfB_s^\ell|)_{t \geq 0}$ is a
    submartingale and Doob's maximal inequality, we have that
    \begin{align}
      \expe{d(\mathbf{H}^M(x), \mathbf{H}^{M+N}(x))^2} &\leq \textstyle{ \sum_{j_1, j_2 \in \nset} 2^{-(j_1+j_2)} \sum_{\ell=M+1}^{M+N} \lambda_\ell e_\ell(x)^2 }\\
      & \qquad \textstyle{ \times \expe{\sup_{t \in \ccint{0,j_1}} |\int_0^t g(s) \rmd \bfB_s^\ell| \sup_{t \in \ccint{0,j_2}} |\int_0^t g(s) \rmd \bfB_s^\ell|}} \\
                                                       &\leq \textstyle{ \sum_{j_1, j_2 \in \nset} 2^{-(j_1+j_2)} \sum_{\ell=M+1}^{M+N} \lambda_\ell e_\ell(x)^2 \expe{\sup_{t \in \ccint{0,\max(j_1,j_2)}} |\int_0^t g(s) \rmd \bfB_s^\ell|^2}} \\
                                                       &\leq 2 \textstyle{ \sum_{j_1, j_2 \in \nset} 2^{-(j_1+j_2)} \sum_{\ell=M+1}^{M+N} \lambda_\ell e_\ell(x)^2 \expe{(\int_0^{\max(j_1,j_2)} g(s) \rmd \bfB_s^\ell)^2}} \\
      &\leq \textstyle{2 \sum_{\ell=M+1}^{\infty} \lambda_\ell e_\ell(x)^2 \sum_{j_1, j_2 \in \nset} 2^{-(j_1+j_2)} G(\max(j_1,j_2)) .}
    \end{align}
    Since $G$ is bounded and $\sum_{j_1, j_2 \in \nset} 2^{-(j_1+j_2)}  < +\infty$ and
    $\lim_{M \to + \infty} \sum_{\ell=M+1}^{\infty} \lambda_\ell e_\ell(x)^2 =
    0$ using Mercer's theorem, we get that for any $\vareps >0$, there exists
    $M \in \nset$ such that for any $N \in \nset$
    \begin{equation}
      \textstyle{\expe{d(\mathbf{H}^M(x), \mathbf{H}^{M+N}(x))^2} \leq \vareps^3 . }
    \end{equation}
    Therefore, combining this result and Markov's inequality, we get that for any
    $\vareps >0$, there exists $M \in \nset$ such that for any $N \in \nset$
    \begin{equation}
      \textstyle{\mathbb{P}(d(\mathbf{H}^M(x), \mathbf{H}^{M+N}(x)) \geq \vareps) \leq \vareps .}
    \end{equation}
    Hence, $(\mathbf{H}^M(x))_{M \in \nset}$ is tight. In addition, we have that
    for each limit point $\tilde{\mathbf{H}}(x)$, for any $t_1, \dots, t_p$ with
    $t_1, \dots t_p \geq 0$ and $p \in \nset$,
    $(\tilde{\mathbf{H}}_{t_1}(x), \dots, \tilde{\mathbf{H}}_{t_p}(x))$ is
    Gaussian with same mean and covariance matrix as
    $(\mathbf{H}_{t_1}(x), \dots, \mathbf{H}_{t_p}(x))$. Therefore, we have that
    $(\mathbf{H}^M(x))_{M \in \nset}$ converges to $\mathbf{H}(x)$, which
    concludes the proof. 
\end{proof}

In what follows, we denote $\{\bfY_t^M(\x)\}_{t \geq 0, x \in \mathcal{X}}$, the
reconstructed process given by
\begin{equation}
  \textstyle{\bfY_t^M(\x) = \mu(x) + \sum_{m=0}^M \lambda_m^{1/2} \bfZ_t^m e_m(x) , }
\end{equation}
where
\begin{equation}
  \textstyle{
    \bfZ_t^m = \rme^{-\int_0^t \beta_s \rmd s / 2} \bfZ_0^m + \rme^{-\int_0^t \beta_s \rmd s / 2}\int_0^t g(s) \rmd \bfB_s^m .
    }
  \end{equation}

  We are now ready to state the main result of this section.

\begin{proposition}{}{}
  For any $M \in \nset$ we have
  \begin{equation}
    \textstyle{
      \bfY_t^M(\x) = \mu(x) + \rme^{-\int_0^t \beta_s \rmd s / 2} \sum_{m=0}^M \lambda_m^{1/2}  \bfZ_0^m e_m(\x) + \mathbf{H}_t^M(\x) .
      }
    \end{equation}
    In addition, denote
    $\eta^M \in \mathcal{P}(\rmc(\coint{0,+\infty}, \rset)^{\mathcal{X}})$ the
    distribution of $\bfY^M$. Then $\lim_{M \to + \infty} \eta^M = \eta$, where
    $\eta \in \mathcal{P}(\rmc(\coint{0,+\infty}, \rset)^{\mathcal{X}})$ is the
    distribution of $\bfY$ defined in \eqref{eq:definition_Y}.
\end{proposition}

\section{Experimental Details}
\label{sec:experimental-details}

In the following sections we provide details on our experimental procedures.
The models and experiments have been implemented in Jax~\citep{jax2018github}.

\subsection{Software and Data}

The core set of tools in Python \citep{van1995python} enabled this work, including Jax \citep{jax2018github}, Optax \citep{deepmind2020jax}, Haiku \citep{haiku2020github}, Hydra \citep{Yadan2019Hydra}, Numpy \citep{harris2020}, Scipy \citep{2020SciPy-NMeth} and Matplotlib \citep{Hunter2007}.
\subsection{1D Stochastic Processes}

\paragraph{Quadratic} A synthetic dataset constructed to exhibit clear bi-modality. Samples consist of function evaluations of \(f(x) := a x^2 + \epsilon\) with \(a \sim \text{Unif}(\{-1,1\})\) and \(\epsilon \sim \text{N}(0,10)\) on a uniformly sampled grid of 100 evaluation points in \([-10,10]\). We generate a dataset of $5000$ samples which we scaled to have unit variance and split in the ratio of $[0.8,0.1,0.1]$ for training, validation and testing respectively.

\paragraph{Melbourne} A real dataset recording the number of pedestrians on streets in Melbourne, each sample being a period of $24$ hours with readings taken hourly (giving a grid of $24$ evaluation points per sample). The dataset is sourced from \url{http://www.timeseriesclassification.com/description.php?Dataset=MelbournePedestrian}. We pre-processed the dataset by removing rows with unobserved values and rescaling the dataset to have unit variance. The pre-processed dataset contains 2319 samples and we split the dataset in the ratio of \([0.8, 0.1, 0.1]\) for training, validation and testing respectively.

\paragraph{Gridwatch} A real dataset recording the energy demand of the UK grid, each sample being a period of $24$ hours with readings taken every five minutes (giving a grid of $288$ evaluation points per sample). The dataset is sourced from \url{https://www.gridwatch.templar.co.uk/download.php} by selecting the `demand' field. We performed a number of pre-processing steps:
\begin{itemize}
    \item Select days where the readings were taken exactly on the fifth minute,
    \item Remove days where the difference between two consecutive readings was beyond the 99.5\textsuperscript{th} percentile of all differences,
    \item Remove days where consecutive readings did not vary for half an hour or more,
    \item Centre the dataset so each sample has zero mean,
    \item Scale the dataset so each sample has unit variance.
\end{itemize}
This resulted in $1013$ samples which we split in the ratio of \([0.8, 0.1, 0.1]\) for training, validation and testing respectively. 

\subsection{1D Gaussian Process Regression}
For the 1D regression tasks we generate synthetic Gaussian process datasets by sampling from GP posteriors with EQ and Matern kernels. The sampling process is as follows:
\begin{enumerate}
    \item Sample $100$ \(x\) values uniformly from \([-2, 2]\),
    \item Draw a single sample from the prior GP at the sampled \(x\) values,
    \item Sample the number of context points uniformly between \(1\) and \(50\),
    \item Sample the context pairs uniformly from the prior GP samples,
    \item Compute the posterior distribution given the context pairs and draw a single sample from this posterior.
\end{enumerate}
The evaluation set is taken to be the entire \(100\) \(x\) values. The particular kernels we consider are the Exponentiated Quadratic (EQ) and the Matern \(5/2\), using lengthscale \(1.00\) and variance \(1.00\).

\subsection{Methods}

\paragraph{Gaussian Processes} We used a Matern-1/2 kernel with lengthscale and noise variance learnt by maximising the marginal likelihood. 

\paragraph{Neural Processes} For the unconditional modelling of stochastic processes we modified a PyTorch implementation of Neural Processes available online at \url{https://github.com/EmilienDupont/neural-processes}. Our modifications adapted the implementation to unconditional training/sampling by implementing Equation 7 of \cite{garnelo2018neural}. We set the dimension of the representation of context points and the dimension of the latent variable to be $512$ ($1024$ for Gridwatch dataset). The encoder and decoder were three-layered FCNs with $512$ ($1024$ for Gridwatch dataset) units and ReLU activations, trained to $1000$ epochs.

For the predictive task on 1D datasets and the NP, ANP and AGNP benchmarks on the 1D Gaussian process regression task (Section \ref{sec:1d_gp_regression}), we used the implementation available online at \url{https://github.com/wesselb/neuralprocesses}. The architectural choices for NPs, ANPs and AGNPs are taken from \citet{markou2022Practical}, which are as follows:
\begin{itemize}
    \item NPs: the encoder consists of a deterministic path of $6$ layers with $128$ units and a stochastic path of $2$ layers with $128$ units, each with a mean aggregation layer. The decoder consists of $1$ layer of $128$ units. The dimension of the latent variable was $128$.
    \item ANPs: same as NPs with a dot product self attention layer \citep{vaswani2017attention} with $8$ heads in place of the mean aggregation layer.
    \item AGNPs: the encoder consists of the same deterministic path as the ANP and there is no stochastic path. $512$ basis functions are used to parametrise the predictive covariance.
\end{itemize}
Each model is trained using an Adam optimiser \citep{kingma2014method} with learning rate \(5\times10^{-4}\) for $100$ epochs. Each epoch consists of $1024$ batches of $8$ tasks.

\paragraph{Spectral Diffusion Models}

The architecture of the score network $\bm{s}_\theta$ is given by a multilayer
perceptron with $6$ hidden layers with $512$ units each.  We use sinusoidal activation functions. For the context embedding network $c = \textrm{enc}_\theta(\{x^i,y^i\}_{i \in C})$ we use $4$ layers of the bi-dimensional attention block proposed by \citet{dutordoir2022neural} with $4$ self attention heads and $128$ units each. We aggregate the output of the bi-dimensional attention layers using mean aggregation over the number of context points and the process dimension, before finally passing through a single output layer of $128$ units with GELU activation.

All models are trained by the stochastic optimiser Adam \citep{kingma2014method}
with parameters $\beta_1=0.9$, $\beta_2=0.999$, batch-size of $512$ data-points.
The learning rate is annealed with a linear ramp from $0$ to $5000$ steps, reaching the maximum value of \(2\times10^{-4}\), and from then with a cosine schedule down to $0$ after $50000$ iterations in total.

Following \citet{song2020score}, the diffusion model's
diffusion coefficient is parameterised as $g(t) = \sqrt{\beta(t)}$ with
$\beta: \ t \mapsto \beta_{\min} + (\beta_{\max} - \beta_{\min}) \cdot t$, 
where we found $\beta_{\min}=0.1$ and $\beta_{\max}=8$ to work best.

\subsection{Functional MMD and two-sample test power}

Our quantitative results are based on the functional MMD and kernel two-sample hypothesis test of \citet{wynne2022kernel}. 

The hypothesis test considered uses the null hypothesis that samples from the model and dataset are from the same distribution, while the alternative states that they are not. We know that the null is false as the distribution of samples from the model will never truly match the data distribution. Therefore, we expect a perfectly powerful test to always reject the null hypothesis. By choosing a specific test with finite power (for example by restricting the number of samples used when computing the test statistic), we can compare the quality of samples from different models by computing the power of the specific test on each of the sets of samples. If the test exhibits a lower power on a set of samples, it indicates those samples are harder to distinguish from the data distribution and thus can be considered of higher quality.

The specific test we choose is the kernel two-sample test with the ID kernel computed on 10 samples. We perform 1000 tests to estimate the power. Confidence intervals are obtained by repeating the procedure on models trained from different seeds.

\subsection{Additional Samples}
\label{app:additional_samples}

In Figure \ref{fig:1d_cond_1_5} we present predictive samples on 1D datasets with small context sets.

\begin{figure}[h]
     \centering
     \begin{subfigure}[b]{0.48\textwidth}
         \centering
         \includegraphics[width=\textwidth]{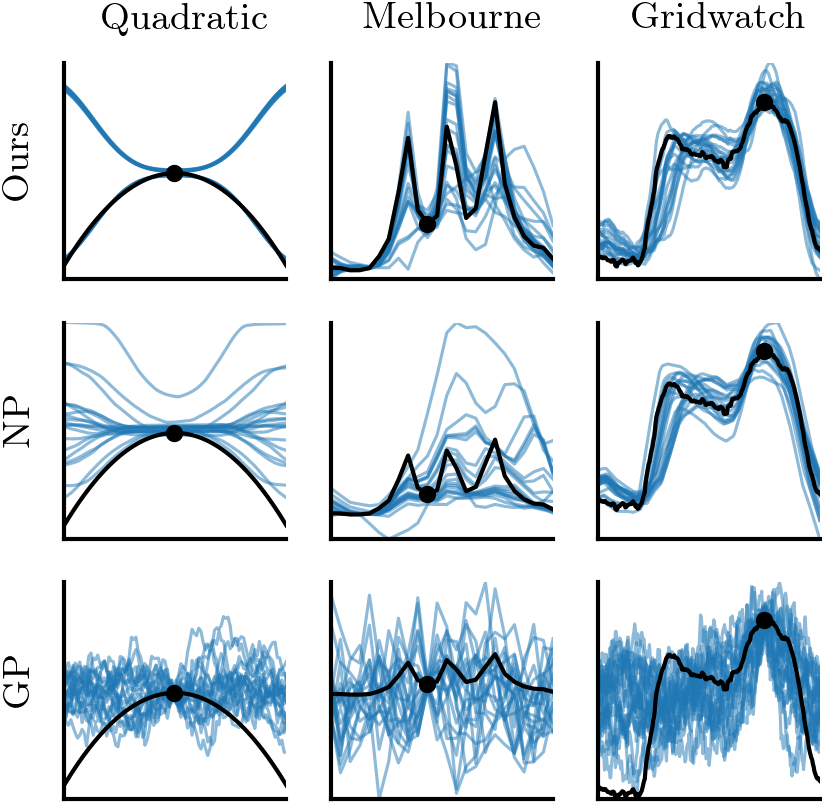}
     \end{subfigure}
     \hfill
     \begin{subfigure}[b]{0.48\textwidth}
         \centering
         \includegraphics[width=\textwidth]{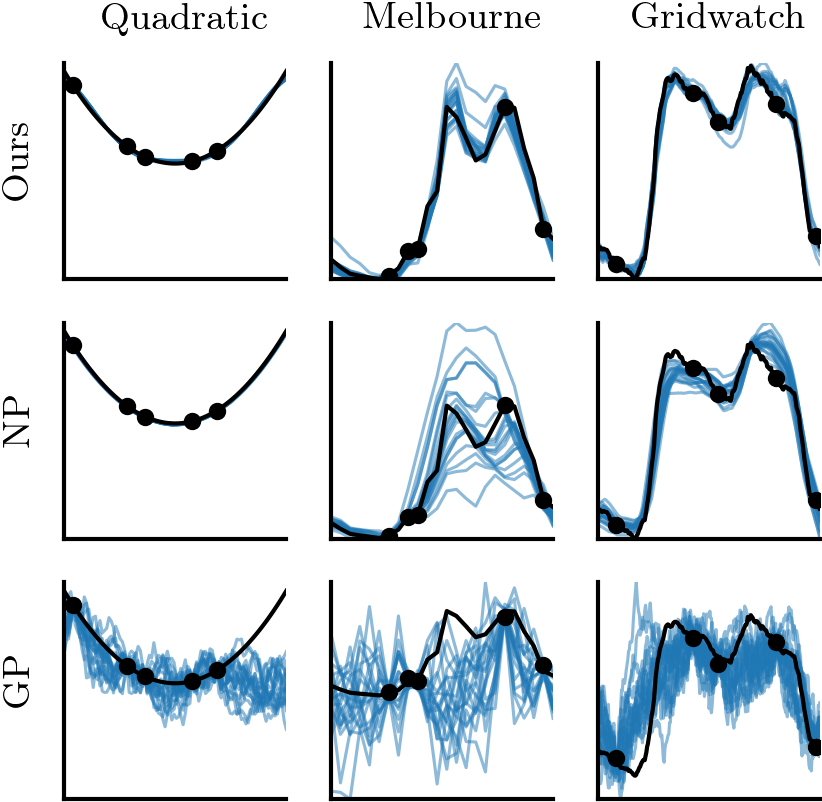}
     \end{subfigure}
        \caption{Predictive samples from our model, NPs and GPs conditioned on 1 (left) and 5 (right) context points.}
        \label{fig:1d_cond_1_5}
\end{figure}

In \Cref{fig:mnist_samples} we visualise some uncurated samples from our model on MNIST.

\begin{figure}
    \centering
    \includegraphics[width=0.5\textwidth]{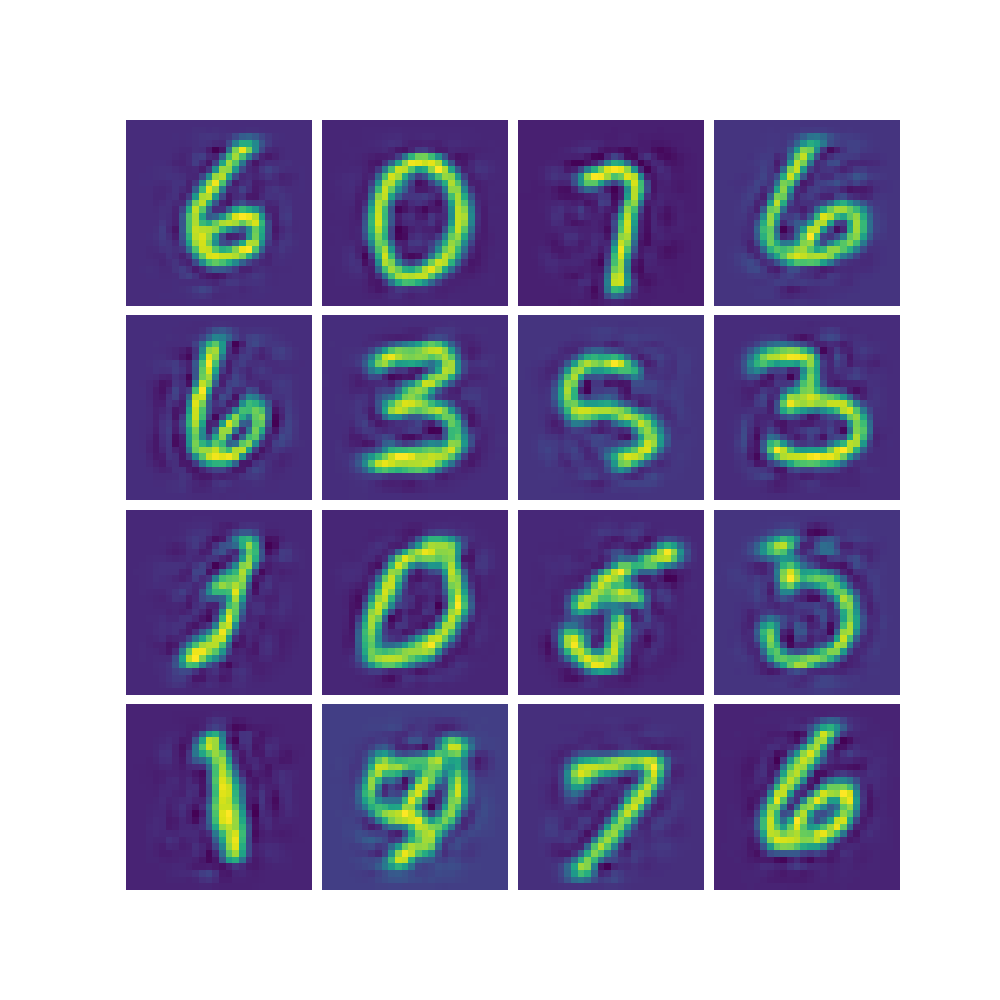}

    \caption{Uncurated MNIST samples, covariance kernel, \(M=100\).}
    \label{fig:mnist_samples}
\end{figure}

\subsection{Further Ablation Studies}

In \Cref{fig:rbf_examples,fig:rbf_noising_trajectory,fig:matern_examples,fig:matern_noising_trajectory} we demonstrate the effect of different analytic kernels on the quality of MNIST samples from the Spectral Diffusion Model.

\begin{table}
    \centering
    \setlength{\tabcolsep}{0pt}
    \begin{tabular}{c c}
        \includegraphics[width=.3\textwidth]{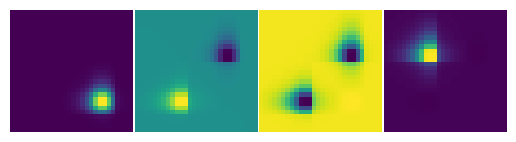} & \includegraphics[width=.3\textwidth]{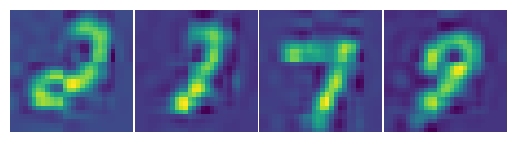} \\
        \includegraphics[width=.3\textwidth]{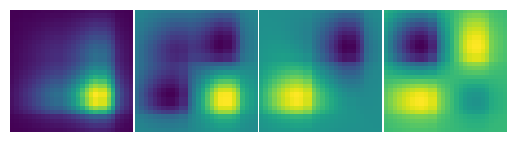} & \includegraphics[width=.3\textwidth]{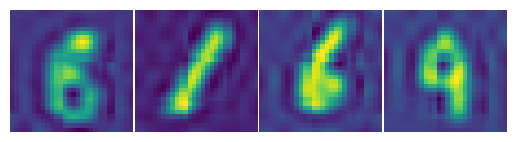} \\
        \includegraphics[width=.3\textwidth]{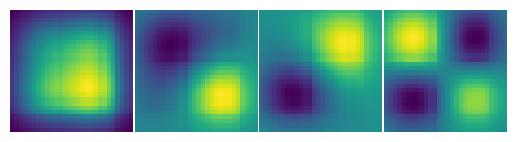} & \includegraphics[width=.3\textwidth]{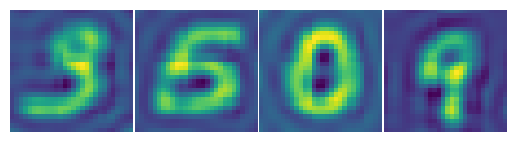} \\
        \includegraphics[width=.3\textwidth]{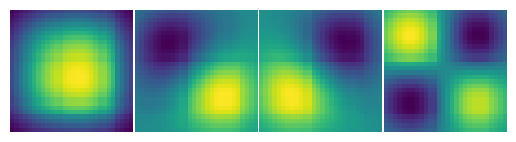} & \includegraphics[width=.3\textwidth]{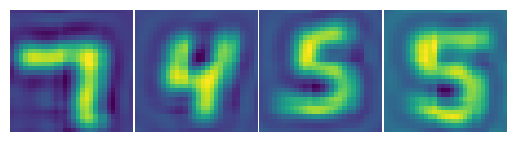}\\
        \includegraphics[width=.3\textwidth]{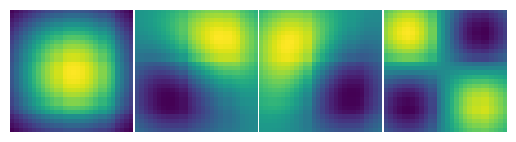} & \includegraphics[width=.3\textwidth]{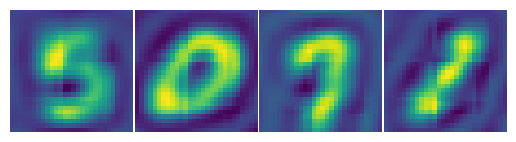} \\
    \end{tabular}
    \captionof{figure}{Eigenfunctions [left] and samples from our model [right] for RBF kernel with lengthscale in [1,2,3,4,5] [top to bottom respectively].}
    \label{fig:rbf_examples}
\end{table}

\begin{figure}[!t]
    \centering
    \raisebox{-0.5\height}{\includegraphics[width=0.76\textwidth]{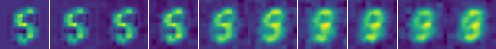}}
    \hfill
    \raisebox{-0.5\height}{\includegraphics[width=0.10\textwidth]{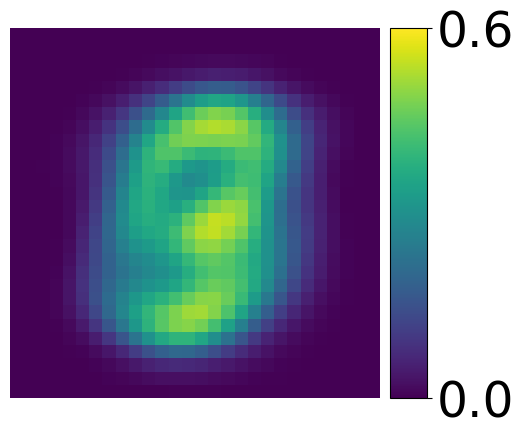}} 
    \hfill
    \raisebox{-0.5\height}{\includegraphics[width=0.10\textwidth]{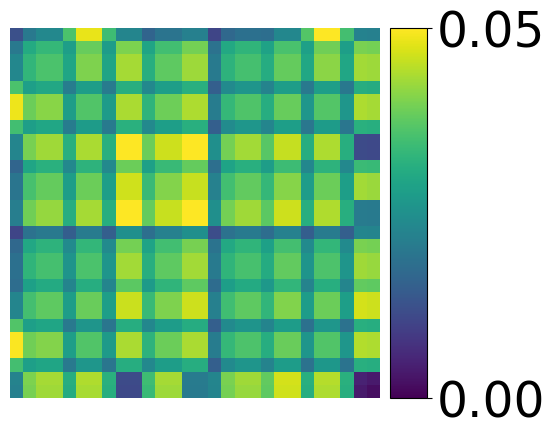}}\\
    \raisebox{-0.5\height}{\includegraphics[width=0.76\textwidth]{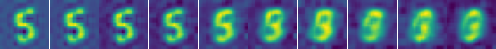}}
    \hfill
    \raisebox{-0.5\height}{\includegraphics[width=0.10\textwidth]{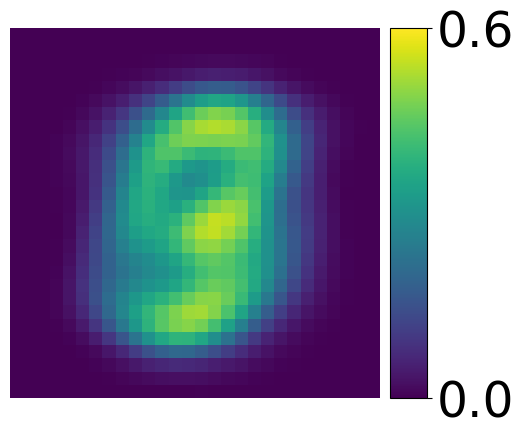}} 
    \hfill
    \raisebox{-0.5\height}{\includegraphics[width=0.10\textwidth]{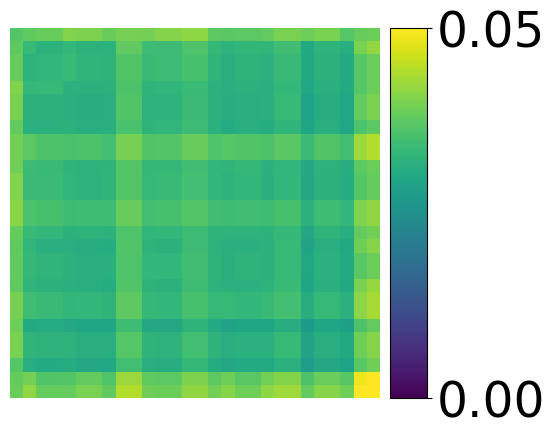}}\\
    \raisebox{-0.5\height}{\includegraphics[width=0.76\textwidth]{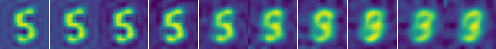}}
    \hfill
    \raisebox{-0.5\height}{\includegraphics[width=0.10\textwidth]{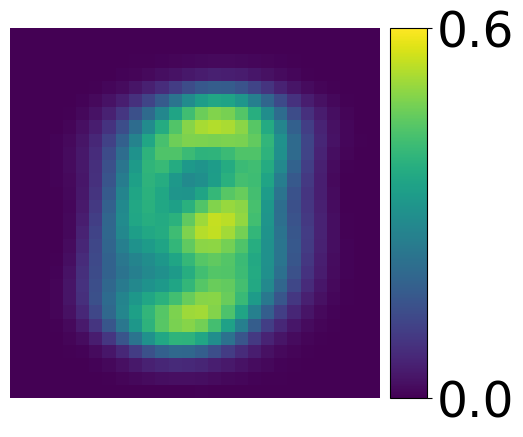}} 
    \hfill
    \raisebox{-0.5\height}{\includegraphics[width=0.10\textwidth]{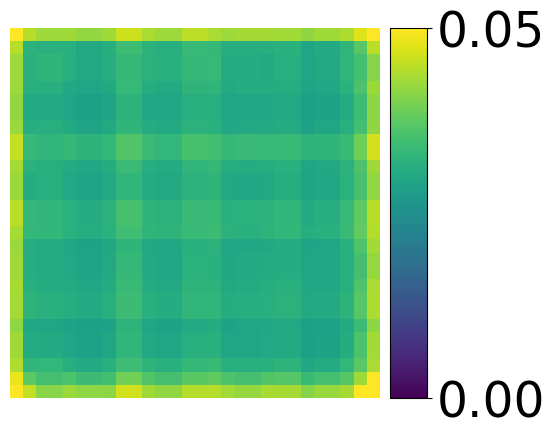}}\\
    \raisebox{-0.5\height}{\includegraphics[width=0.76\textwidth]{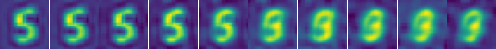}}
    \hfill
    \raisebox{-0.5\height}{\includegraphics[width=0.10\textwidth]{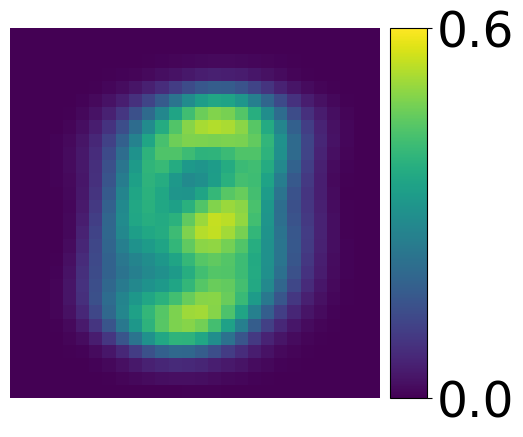}} 
    \hfill
    \raisebox{-0.5\height}{\includegraphics[width=0.10\textwidth]{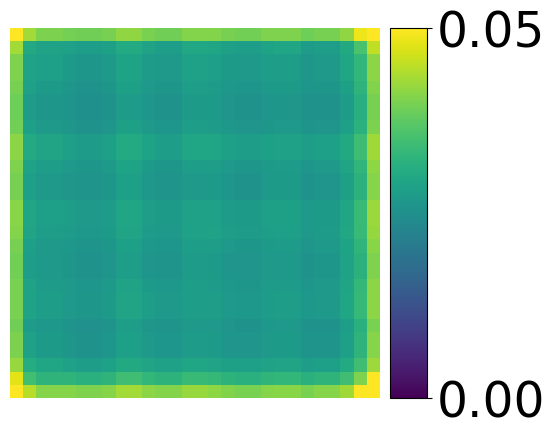}}\\
    \raisebox{-0.5\height}{\includegraphics[width=0.76\textwidth]{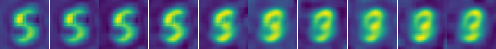}}
    \hfill
    \raisebox{-0.5\height}{\includegraphics[width=0.10\textwidth]{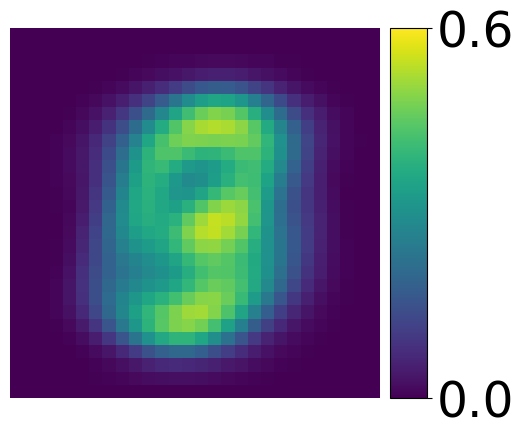}} 
    \hfill
    \raisebox{-0.5\height}{\includegraphics[width=0.10\textwidth]{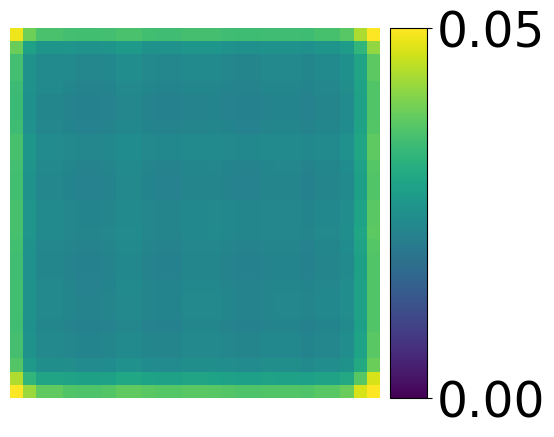}}\\
    \caption{Forward noising process of our model on MNIST digits using RBF kernel with lengthscales in [1,2,3,4,5] [top to bottom respectively]. Rightmost columns show the pixel-wise mean and standard deviation of the reference measure respectively.}
    \label{fig:rbf_noising_trajectory}
\end{figure}

\begin{table}
    \centering
    \setlength{\tabcolsep}{0pt}
    \begin{tabular}{c c}
        \includegraphics[width=.3\textwidth]{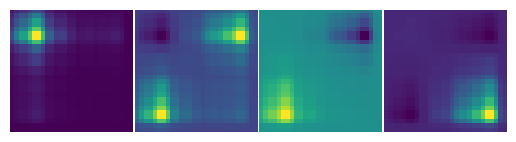} & \includegraphics[width=.3\textwidth]{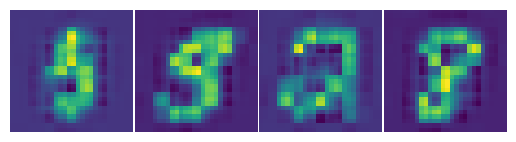} \\
        \includegraphics[width=.3\textwidth]{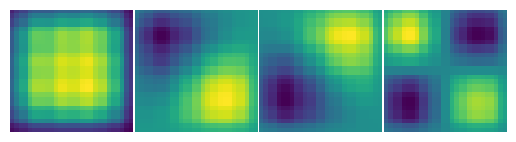} & \includegraphics[width=.3\textwidth]{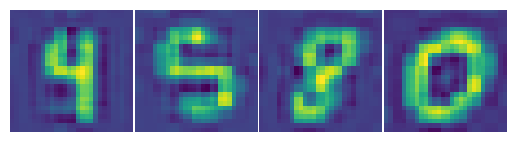} \\
        \includegraphics[width=.3\textwidth]{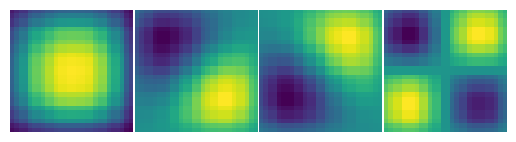} & \includegraphics[width=.3\textwidth]{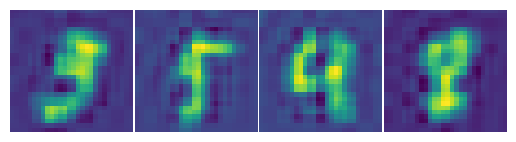} \\
        \includegraphics[width=.3\textwidth]{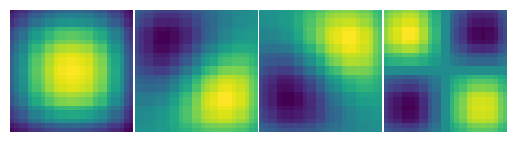} & \includegraphics[width=.3\textwidth]{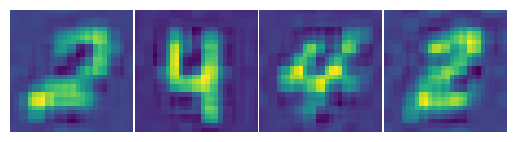}\\
        \includegraphics[width=.3\textwidth]{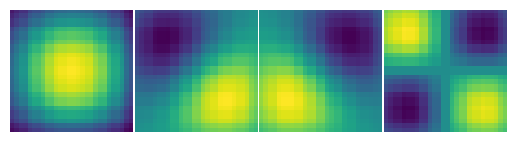} & \includegraphics[width=.3\textwidth]{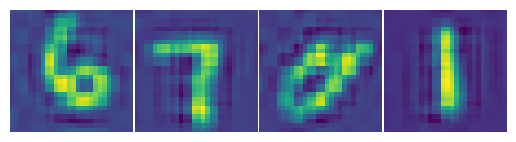} \\
    \end{tabular}
    \captionof{figure}{Eigenfunctions [left] and samples from our model [right] for Matern-3/2 kernel with lengthscale in [1,2,3,4,5] [top to bottom respectively].}
    \label{fig:matern_examples}
\end{table}

\begin{figure}
    \centering
    \raisebox{-0.5\height}{\includegraphics[width=0.76\textwidth]{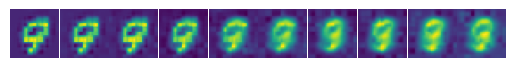}}
    \hfill
    \raisebox{-0.5\height}{\includegraphics[width=0.10\textwidth]{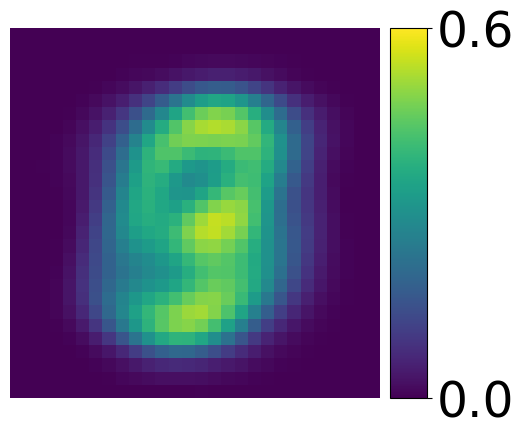}} 
    \hfill
    \raisebox{-0.5\height}{\includegraphics[width=0.10\textwidth]{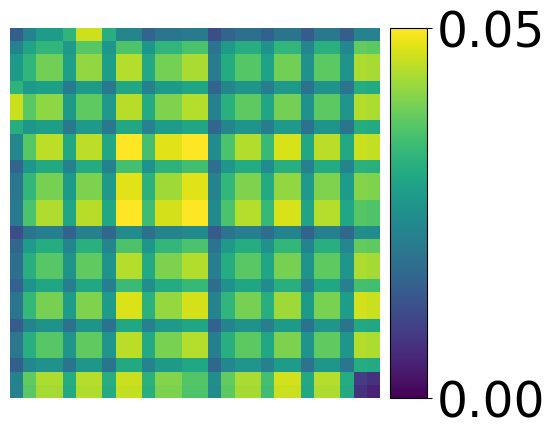}}\\
    \raisebox{-0.5\height}{\includegraphics[width=0.76\textwidth]{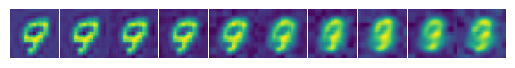}}
    \hfill
    \raisebox{-0.5\height}{\includegraphics[width=0.10\textwidth]{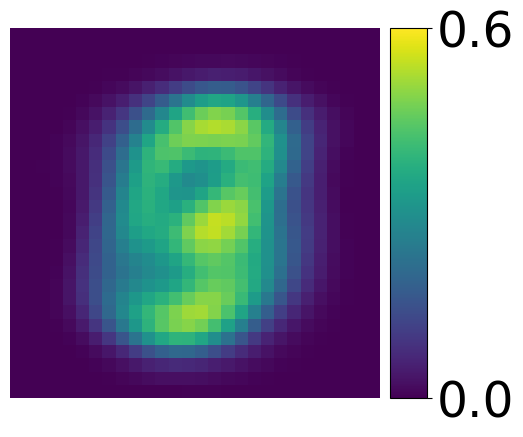}} 
    \hfill
    \raisebox{-0.5\height}{\includegraphics[width=0.10\textwidth]{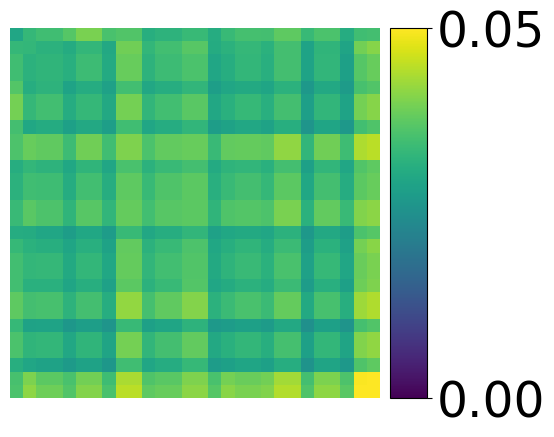}}\\
    \raisebox{-0.5\height}{\includegraphics[width=0.76\textwidth]{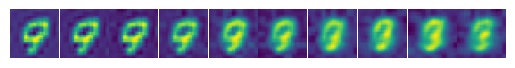}}
    \hfill
    \raisebox{-0.5\height}{\includegraphics[width=0.10\textwidth]{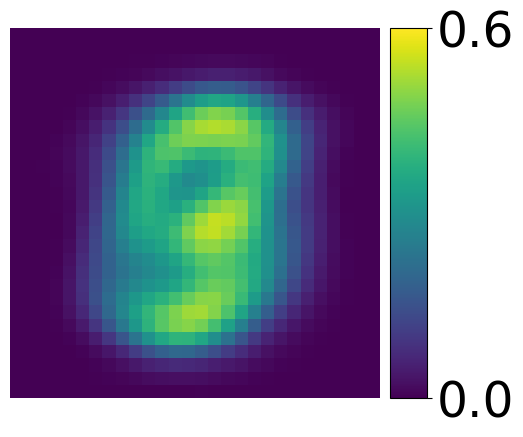}} 
    \hfill
    \raisebox{-0.5\height}{\includegraphics[width=0.10\textwidth]{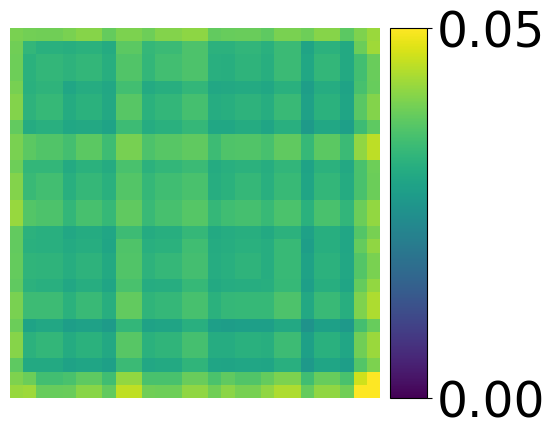}}\\
    \raisebox{-0.5\height}{\includegraphics[width=0.76\textwidth]{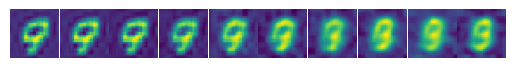}}
    \hfill
    \raisebox{-0.5\height}{\includegraphics[width=0.10\textwidth]{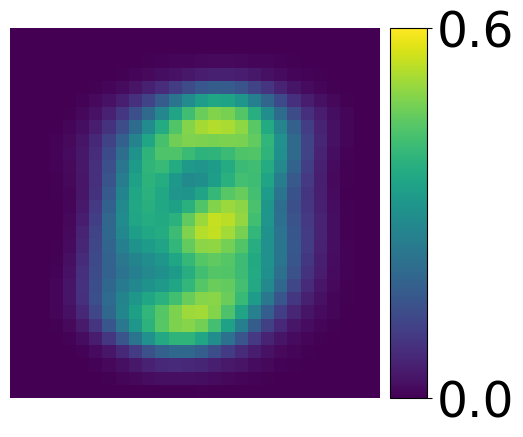}} 
    \hfill
    \raisebox{-0.5\height}{\includegraphics[width=0.10\textwidth]{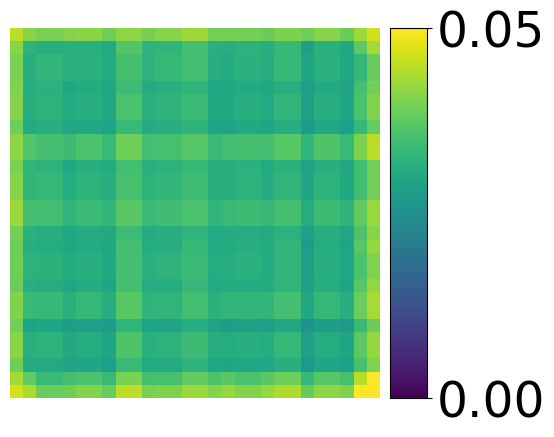}}\\
    \raisebox{-0.5\height}{\includegraphics[width=0.76\textwidth]{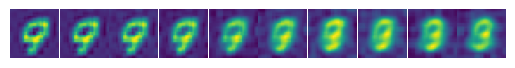}}
    \hfill
    \raisebox{-0.5\height}{\includegraphics[width=0.10\textwidth]{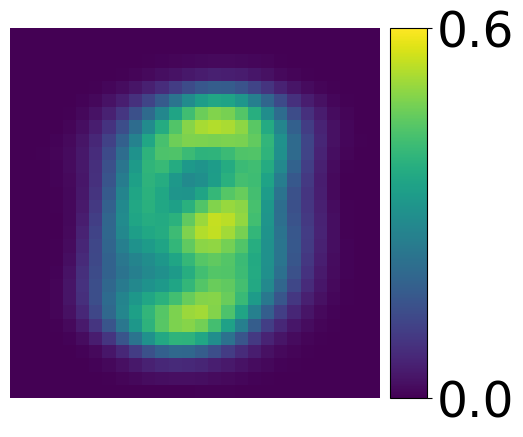}} 
    \hfill
    \raisebox{-0.5\height}{\includegraphics[width=0.10\textwidth]{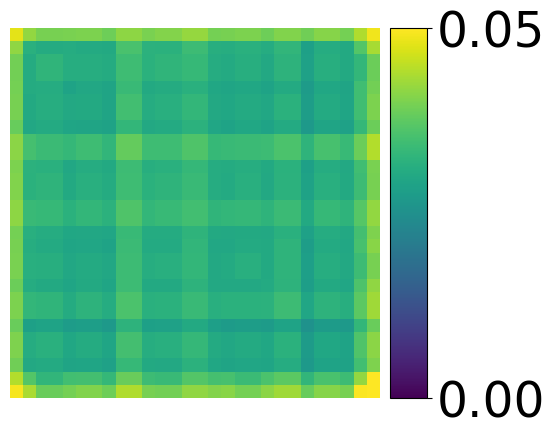}}\\
    \caption{Forward noising process of our model on MNIST digits using Matern-3/2 kernel with lengthscales in [1,2,3,4,5] [top to bottom respectively]. Rightmost columns show the pixel-wise mean and standard deviation of the reference measure respectively.}
    \label{fig:matern_noising_trajectory}
\end{figure}
\newpage 
\newpage

\end{document}